\PassOptionsToPackage{dvipsnames}{xcolor}
\RequirePackage[l2tabu, orthodox]{nag}
\documentclass[12pt]{article}
\usepackage{booktabs} 

\usepackage{amsmath,amsfonts,bm}

\def\eqref#1{equation~\ref{#1}}

\def\1{\bm{1}}

\DeclareMathAlphabet{\mathsfit}{\encodingdefault}{\sfdefault}{m}{sl}
\SetMathAlphabet{\mathsfit}{bold}{\encodingdefault}{\sfdefault}{bx}{n}

\DeclareMathOperator*{\argmin}{arg\,min}

\usepackage[affil-it]{authblk}
\usepackage{todonotes}
\usepackage[parfill]{parskip}
\usepackage[margin=1in]{geometry}
\usepackage[defaultlines=3,all]{nowidow}

\usepackage[T1]{fontenc}
\usepackage{microtype}
\usepackage[english]{babel}
\usepackage{libertine}             
\usepackage{titlesec}
\titleformat{\section}[hang]{\Large\bfseries}{\thesection}{1em}{}
\usepackage{amsmath}
\usepackage{amssymb}
\usepackage{amsthm}
\usepackage{bbm}
\usepackage{mathtools}

\def\11{{\mathbf 1}}    

\def\11{{\mathbf 1}}    

\usepackage{adjustbox}

\usepackage{comment}             
\usepackage{wrapfig}             
\usepackage{caption}             
\usepackage{pifont}              
\usepackage{minitoc}             
\usepackage{algorithm}           
\usepackage[noend]{algpseudocode}
\usepackage{nicefrac}            
\usepackage{booktabs}            
\usepackage{tikz}
\usetikzlibrary{tikzmark}        
\usepackage{stfloats}
\usepackage{bm}
\usepackage[inline]{enumitem}
\usepackage{multirow}
\usepackage[normalem]{ulem}
\usepackage[toc,page,header]{appendix}
\usepackage{minitoc}
\usepackage{centernot}
\usepackage{enumitem}
\usepackage{hyperref}
\usepackage{cleveref}
\usepackage{subcaption}

\usepackage[dvipsnames,svgnames]{xcolor}

\usepackage[sort]{natbib}
\usepackage{bibunits}

\hypersetup{colorlinks,linktoc=all}
\usepackage[all]{hypcap}
\hypersetup{citecolor=Violet}
\hypersetup{linkcolor=black}
\hypersetup{urlcolor=MidnightBlue}

\usepackage{fancyhdr}

\theoremstyle{plain}
\newtheorem{theorem}{Theorem}[section]
\newtheorem{proposition}[theorem]{Proposition}
\newtheorem{lemma}[theorem]{Lemma}

\theoremstyle{definition}
\newtheorem{definition}[theorem]{Definition}
\newtheorem{assumption}[theorem]{Assumption}
\theoremstyle{remark}

\fancypagestyle{FirstPage}{
\lfoot{\textcolor{DarkBlue}{\copyright2025 Rational Intelligence Lab at CISPA. All Rights Reserved}}
}

\title{Incentive Aware AI Regulations: \\ A Credal Characterisation}

\author[1]{Anurag Singh\footnote{Corresponding author: anurag.singh@cispa.de}}
\author[1,2]{Julian Rodemann}
\author[3]{Rajeev Verma}
\author[4]{Siu Lun Chau}
\author[1]{Krikamol Muandet}
\affil[1]{\small{Rational Intelligence Lab, CISPA Helmholtz Center for Information Security, Saarbrücken, Germany}}
\affil[2]{\small{Department of Statistics, LMU Munich, Germany}}
\affil[3]{\small{UvA-Bosch Delta Lab, University of Amsterdam, Netherlands}}
\affil[4]{\small{Nanyang Technological University, Singapore}}
\date{\today} 

\begin{document}

\maketitle

\begin{bibunit}[alp]

\begin{abstract}
    The rapid proliferation of AI applications has intensified debate on effective regulation of these black-box services. Effective regulation must balance two competing goals: (1) deterring non-compliant providers from entering the market, while (2) retaining compliant ones. We call this ideal the \emph{perfect market outcome} (PMO). Regulators face two compounding obstacles that make PMO difficult to achieve: providers hold private information and can act strategically to evade compliance, while any evidence drawn or derived from a finite sample carries statistical uncertainty in proving non-compliance. As this information asymmetry and statistical uncertainty is inherent to any effective regulation, we formalise them through a \emph{mechanism design framework that explicitly accounts for such statistical uncertainty}. This yields a sharp characterisation: a mechanism achieves PMO if and only if the set of non-compliant evidence distributions forms a closed, convex set of probability measures, known in imprecise probability as a \emph{credal set}. This result serves as a diagnostic tool to determine whether PMO is achievable under a given regulation. We further show that PMO-achieving mechanisms can be constructed from a collection of hypothesis tests, and validate our theoretical contributions through experiments on spurious-feature and fairness-based regulations.
\end{abstract}


\doparttoc 
\faketableofcontents 

\section{Introduction}
\label{sec:intro}

As machine learning systems are increasingly deployed in high-stakes domains, ranging from credit scoring~\citep{baesens2003benchmarking} to social justice~\citep{angwin2022machine}, their associated risks can no longer be overlooked~\citep{buolamwini2018gender,laux2024trustworthy}. Policymakers have responded to these risks by developing AI governance frameworks, such as the EU AI Act~\citep{edwards2021eu}. While stricter regulations may deter non-compliance, it also risks triggering market collapse~\citep{goodman2017european}. Effective regulations must therefore balance two competing goals: (1) preventing non-compliant model providers, e.g., those failing fairness or robustness standards, from entering the market, while simultaneously (2) retaining compliant providers within the market. We refer to this situation as the \emph{perfect market outcome} (PMO), in which regulators successfully deter non-compliance without excessively rejecting legitimate entrants. 

Achieving PMO is challenging in practice because AI regulators operate under \emph{information asymmetry}: providers typically retain private information about their systems such as the choice of datasets, model architectures, weights, or training procedures~\citep{kolt2024responsible}. Since the burden of proving non-compliance rests with the regulator, this information asymmetry gives providers considerable room to evade compliance requirements \citep{casper2024black,li2025making}, which regulators may attempt to counter by demanding full access to model weights, gradients, training procedures, and hyper-parameters~\citep{shevlane2022structured,solaiman2023gradient}. However, proprietary interests and trade-secret protections often preclude the disclosure of such private information~\citep{pasquale2015black,raji2020closing}. As a result, policymakers have  argued that practical AI regulation must infer compliance from black-box access alone~\citep{brundage2018malicious}. This compounds information asymmetry with \emph{statistical uncertainty}: without direct access to the model, compliance can only be inferred from its behaviour on a finite evaluation dataset sampled from an underlying population~\citep{ren2024safetywashing,jansen2024statistical}. Each draw yields a different sample of evidence, so any compliance verdict rests on a noisy proxy for the population-level truth: a model judged non-compliant on one sample may well pass on another. To overcome these challenges, policymakers have advocated for  black-box regulatory frameworks that shift the burden of proof onto model  providers~\citep{hadfield2023regulatory,bova2023both}. 

\begin{figure}
    \centering
    \includegraphics[width=\linewidth]{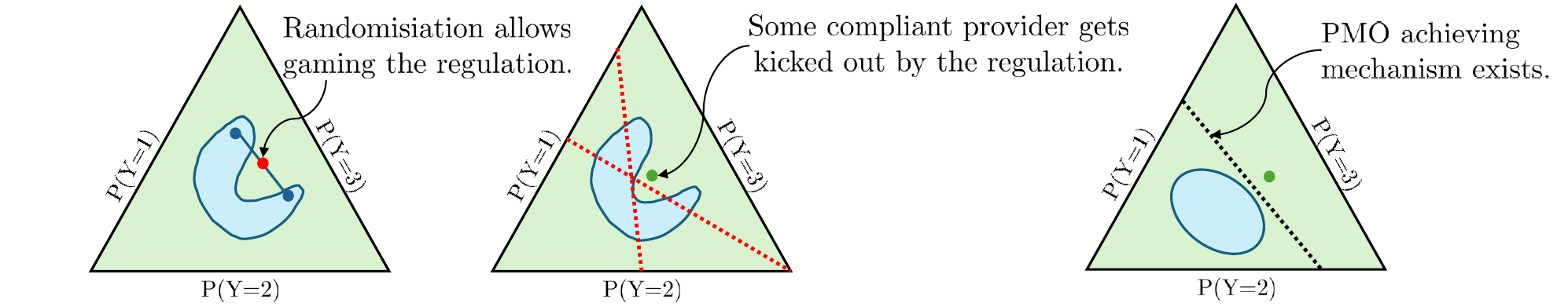}
    \caption{An illustration of our main result, Theorem \ref{theorem:obedience-to-credal}, for a classification task with $K=3$ classes. The blue regions represent the set of non-compliant evidence distributions within the probability simplex and shows how a mechanism can fail when the regulatory requirement results in the set of non-compliant evidence distributions to be non-credal. \textbf{Left:} a non-compliant provider bypasses the regulation by creating a compliant mixture (red dot) from two non-compliant models (dark blue dots). \textbf{Middle:} mechanism kicks out some compliant provider as they cannot be separated from non-compliant providers by a linear functional (dotted red line). \textbf{Right:} when the set of non-compliant evidence distributions is a credal set, a PMO achieving regulation exists.}
    \label{fig:geometry_of_regulation}
\end{figure}

This naturally casts AI regulation as a \emph{mechanism design problem under uncertainty}. Any effective burden-shifting regulation must counter the strategic use of information asymmetry, a classical concern of mechanism design ~\citep{hurwicz1973design,maskin1999nash}. The ``under uncertainty'' qualifier extends the framework to designers who must act on noisy, finite-sample evidence rather than direct observation of provider types. We ask the question, \emph{when can such black-box regulation mechanisms actually achieve their stated goal?} Our main result (Theorem~\ref{theorem:obedience-to-credal}, illustrated in Figure ~\ref{fig:geometry_of_regulation}) gives a sharp answer: \textbf{a PMO-achieving mechanism exists if and only if the set of non-compliant evidence distributions is closed and convex}, commonly known as a \emph{credal set} in the field of Imprecise Probability (IP)~\citep{walley1991}. In plain terms, the requirement must classify any randomised combination of non-compliant distributions as itself non-compliant. This condition serves as a diagnostic tool that tells regulators whether their regulatory requirements can be effectively enforced at all. Additionally, the credal set characterisation in Theorem~\ref{theorem:obedience-to-credal} allows us to leverage tools from IP to repurpose collection of hypothesis tests into PMO-achieving mechanisms.

\textbf{Our Contributions.} We formalise PMO as the design objective of a mechanism design problem under uncertainty, and use this formulation to derive a full characterisation of when an AI regulation can be implemented by a mechanism that achieves PMO in terms of credal sets. Specialising to threshold-based regulations, we further show that quasi-convexity and lower semicontinuity of the threshold function are the characterising conditions. Together, these results give regulators a principled way to check, whether a proposed regulation can be enforced without sacrificing the market. We also show that PMO-achieving mechanisms can be constructed from a collection of hypothesis tests, and validate our theoretical findings with experiments on both synthetic and real-world datasets.

\section{Preliminaries}
This section introduces the notation, presents the problem formulation, and reviews the necessary background on imprecise probabilities and mechanism design.

\textbf{Notation and problem formulation:} We consider a setting with two players: a \emph{model provider} and a \emph{regulator}. The provider trains predictive models and deploys them as a service (e.g., via an API), while the regulator ensures that these services comply with the established regulations. In our supervised learning setup, the provider selects a model $f:\mathcal{X}\rightarrow\mathcal{Y}$ from a hypothesis class $\mathcal{H}$, mapping an input space $\mathcal{X}\subseteq \mathbb{R}^d$ to a target space $\mathcal{Y}$. Model performance is evaluated using a loss function $\ell:\mathcal{Y}\times\mathcal{Y}\rightarrow\mathbb{R}_{\geq 0}$, where $\ell(f(x),y)$ represents the error on a data point $(x,y)\in\mathcal{X}\times\mathcal{Y}$. Let $(\Omega,\mathcal{F})$ be a measurable space associated with a fixed but unknown data-generating process $P$. The corresponding random variables are $X:\Omega\rightarrow\mathcal{X}$ and $Y:\Omega\rightarrow\mathcal{Y}$, with realizations $X(\omega)=x$ and $Y(\omega)=y$. Finally, we denote the space of probability distributions on a given set $A$ by $\Delta(A)$.

\textbf{Evidence space for outcome-based regulation:} In our setting, regulators only have black-box access to model providers' services; internal properties such as model parameters are unobservable, rendering regulations based on them unenforceable. Hence, we define an evidence space $\mathcal{Z}$ with a corresponding random variable $Z:\Omega\rightarrow\mathcal{Z}$, which the regulator uses to specify and evaluate compliance. The evidence space $\mathcal{Z}$ accommodates diverse real-world settings, including $Z:=\ell(f(X(\omega)),Y(\omega))$ for a static model $f$ and $Z:=\ell(f_{\gamma(X(\omega))}(X(\omega)),Y(\omega))$ for services employing a dynamic model router $\gamma:\mathcal{X}\rightarrow\mathcal{H}$~\citep{jitkrittum2025universal}. While randomness in these examples originates purely from the data, it may also arise jointly from the data and a stochastic model, such that $Z=\ell(h(X),Y)$ where $h\in\Delta(\mathcal{H})$. Finally, we equip the space of probability measures $\Delta(\mathcal{Z})$ with the weak-* topology induced by $\mathcal{C}(\mathcal{Z})$ which denotes the space of bounded continuous functions on $\mathcal{Z}$.

\subsection{Imprecise Probabilities (IP) and Credal Sets}
Standard probability theory~\cite{kolmogorov1956foundations} assigns a unique number in $[0,1]$ to every event in $\mathcal{F}$. In contrast, \emph{imprecise probability} (IP) generalizes this framework to accommodate ambiguity, partial ignorance, or conflicting evidence by allowing a range of plausible probability distributions~\citep{walley1991, augustin_introduction_2014}. While classical probability represents uncertainty using a single (additive) probability distribution $P$, various IP models such as lower probabilities, possibility measure, and belief functions are instead characterised by sets of distributions, commonly referred to as \emph{credal} sets.
\begin{definition}[Credal Set] 
A \emph{credal set} $\mathcal{P}_0$ is a closed, convex set of probability measures.
\end{definition}
This extension is historically grounded in the subjective interpretation of probability~\citep{definetti1974theory}, which departs from frequentist views by interpreting probability as an agent's betting dispositions rather than observed frequencies. From a robust Bayesian perspective, a credal set represents the agent's uncertainty: the ``true'' or ``ideal'' data-generating distribution is assumed to lie within $\mathcal{P}_0$, although its exact identity remains unknown. Central to this interpretation is the concept of a \emph{gamble}. A gamble $g: \Omega \rightarrow \mathbb{R}$ is a bounded real-valued function interpreted as an uncertain reward whose payoff is $g(\omega)$ if the outcome $\omega \in \Omega$ occurs. The precise probability $P(A)$ for an event $A\in\mathcal{F}$ coincides with the \emph{fair price} at which an agent is willing to buy or sell the associated indicator gamble $\mathbf{1}_A$ defined by $\mathbf{1}_A(\omega)=1$ when $\omega\in A$ and $0$ otherwise. To accommodate imprecision, the IP literature replace a single fair price with bounds: the supremum acceptable buying price and infimum acceptable selling price of a gamble. The two are conjugate and coincide in the precise case, characterizing agent's beliefs via the set of risks they are willing to accept, known as the set of \emph{marginally desirable gambles}. See \citet{augustin_introduction_2014} for more details.
\begin{definition}[Set of Marginally Desirable Gambles]
A gamble $g$ is \emph{marginally desirable} with respect to a credal set $\mathcal{P}_0$ if the agent expects a non-negative gain in the worst-case scenario. A set of marginally desirable gambles with respect to $\mathcal{P}_0$ is formally defined as
\begin{align*}
\mathfrak{G}_{\geq 0,\mathcal{P}_0} := \left\{ g: \Omega \to \mathbb{R} \mid \inf_{P \in \mathcal{P}_0} \mathbb{E}_P[g] \geq 0 \right\}.
\end{align*}
\end{definition} 
We denote the set of \emph{marginally undesirable gambles} as $\mathfrak{G}_{\leq 0,\mathcal{P}_0}:=\left\{-g\mid g\in\mathfrak{G}_{\geq 0,\mathcal{P}_0}\right\}$ and the set of desirable gambles w.r.t all distributions in $\Delta(\Omega)$ as $\mathfrak{D}_{\geq 0}$. Gambles serve as a geometric dual to the credal set. In the context of regulation, $\mathcal{P}_0$ represents the uncertainty in the model space, while $\mathfrak{G}_{\leq 0}$ represents the gambles in the evidence space. This relationship is fundamental to understanding actuarial risk, where regulation can be framed as checking whether a specific financial position (gamble) is desirable (acceptable) under a set of plausible stress-test scenarios (the credal set).

\subsection{Mechanism Design}
\label{subsec:mechanism-design}
We briefly review the relevant mechanism design concepts needed to frame AI regulation as a mechanism design problem (see \citet{roughgarden2010algorithmic} or \citet{nisan2007introduction} for a textbook introduction). Mechanism design studies how a designer can implement a desired outcome when an agent holds private information. An agent's private information, typically referred to as \textit{type}, is denoted by $\theta \in \Theta$ where $\Theta$ is the type space. The desired outcome is described by a \emph{social choice function} 

\begin{definition}[Social Choice Function]
Let $\mathcal{O}$ be a space of outcomes (or allocations). A \textit{social choice function (SCF)} $f: \Theta \to \mathcal{O}$ maps the type of an agent $\theta\in\Theta$ to a desired outcome $f(\theta)\in \mathcal{O}$. 
\end{definition}

Since the true type of an agent is unobservable to the designer, they cannot directly apply $f$ to allocate an outcome. 
Also, when asked directly, the agent may strategically misreport it for a better outcome. To address this, the designer designs a \textit{mechanism} $\mathcal{M} = \langle S, g \rangle$ consisting of a strategy space $S$ and an outcome function $g:S \to \mathcal{O}$ that maps strategy profiles to outcomes. Since an agent is self interested, they choose their own best strategy $s^*\in S$. Typically, a designer seeks to design an \emph{implementable mechanism} $\mathcal{M}$.
\begin{definition}[Implementable Mechanism]
A mechanism $\mathcal{M}$ is said to \textit{implement} the SCF $f$ if $g(s^*)$ corresponds to $f(\theta)$ for all agents. 
\end{definition}
An implementable mechanism $\mathcal{M}$ ensures that, for every agent, the outcome associated with the best strategy, i.e., $g(s^*)$ coincides with the desired social outcome $f(\theta)$. A fundamental result in mechanism design is the Revelation Principle \citep{gibbard1973manipulation, myerson1979incentive}: if there exists an implementable mechanism for a SCF, then the SCF can also be implemented by a \textit{direct mechanism}, where the strategy space is the type space itself ($S = \Theta$). This principle allows us to restrict our attention to direct mechanisms without loss of generality. In the case of AI regulation, the designer corresponds to the regulator and the agents to model providers. We therefore restricted our review to single-agent mechanisms, as our formulation assumes that a provider’s outcome does not depend on the actions of other providers. More generally, mechanism design allows for settings in which the SCF and agents’ utilities could depend on strategies of other agents. In auctions \citep{vickrey1961counterspeculation}, for example, competition could be exploited to achieve desired outcome.

\section{Incentive Aware Regulation}

In the context of AI regulation, an agent's (model provider) type $\theta$ corresponds to an element of the space of evidence distributions i.e. $P\in\Delta(Z)$. The outcome space $\mathcal{O} = \{0, 1\}$ represents market participation ($1$ for participation, $0$ for self-exclusion). A regulation can then be formally defined as

\begin{definition}[Requirement]Let $\mathfrak{R}: \Delta(\mathcal{Z}) \rightarrow \{0,1\}$ be a requirement function. An evidence distribution $P \in \Delta(\mathcal{Z})$ satisfies the regulatory requirement if $\mathfrak{R}(P)=1$.
\label{def:e-safe_model}
\end{definition}
Given our focus on outcome-based regulation, Definition~\ref{def:e-safe_model} directly defines the requirements on the observable evidence. In many scenarios, they are defined by thresholding a quantifiable metric $r(P): \Delta(\mathcal{Z}) \rightarrow \mathbb{R}$ such as accuracy, fairness, or worst-case subgroup performance, i.e., $\mathfrak{R}_\tau(P) = \mathbbm{1}[r(P) > \tau]$ where $\tau$ is a pre-defined threshold. The requirement induces an SCF $f$  whose outcome for a model provider with type $P\in\Delta(\mathcal{Z})$ is $\mathfrak{R}(P)$. While in mechanism design types are typically scalar, in our setting, the type corresponds to the entire evidence-generating process $P\in\Delta(\mathcal{Z})$. This makes AI regulation a mechanism design under statistical uncertainty problem.


\subsection{Regulation Mechanisms}

We define a regulation mechanism $\Pi \subseteq \mathcal{C}(\mathcal{Z})$ as a set of non-negative continuous bounded functions. In practice, a possible interpretation of the mechanism $\Pi$ could be a set of \emph{licenses}. A license would map the evidence $z\in\mathcal{Z}$ to a maximum price $\pi(z)$ at which a provider can sell their service. We assume that $\|\pi\|_{\infty}\leq R$ for all $\pi \in \Pi$, which imposes a ``market cap'' on the possible pricing in the market. Let $C<R$ be the market entry fee. 
The mechanism $\Pi$ enforces the requirement $\mathfrak{R}$ by implementing the perfect market outcome. Next, we define the notion of \emph{obedience}.
\begin{definition}[Obedience]
\label{def:obidience_to_regulation}
A regulation mechanism $\Pi$ is said to enforce \emph{obedience} to the requirement $\mathfrak{R}$ if the following holds true ex-ante for the agents: For all $P\in\Delta(\mathcal{Z})$ where $\mathfrak{R}(P)=0$,
\begin{equation}
    \sup_{\pi\in\Pi}\,\mathbb{E}_{Z\sim P}[\pi(Z)]\leq C.
\end{equation}
\end{definition}
Obedience ensures that the non-compliant providers cannot recover their entry fee from any license in $\Pi$\footnote{Another example of $\Pi$ is a contract that penalizes providers based on evidence. Then $R$ would be the maximum penalty and $C$ an initial credit tied to the fee. Definition~\ref{def:obidience_to_regulation} would then flip accordingly: for any $\pi \in \Pi$, the expected penalty for a non-compliant provider exceeds $C$, i.e., $\inf_{\pi \in \Pi} \mathbb{E}_P[\pi] \ge C$. We focus on the license interpretation of $\Pi$ for consistency with prior work on principal-agent hypothesis testing~\citep{bates2022principal,bates2023incentive,hossain2025strategic}.}, and therefore self-exclude. Furthermore, regulations must also be \emph{feasible}, i.e., 
\begin{definition}[Feasibility]
\label{def:feasibility-of-regulation} A regulation mechanism $\Pi$ is \emph{feasible} if for all $P\in\Delta(\mathcal{Z})$ such that $\mathfrak{R}(P)=1$, there exists a license $\pi \in \Pi$ for which $\mathbb{E}_{Z\sim P}[\pi(Z)]-C > 0$.
\end{definition}
Feasibility incentivises participation from the compliant providers. We emphasize that both obedience and feasibility guarantees operate at the level of incentives rather than direct enforcement: \emph{they do not physically restrict deployment or exclude providers, but instead induce compliance through expected outcomes}, implicitly assuming rational behaviour by providers. For example, obedience implicitly assumes that any rational non-compliant provider would not participate in a bet that has negative expected outcome. Next, we formalize the notion of perfect market outcomes in terms of model providers' ex ante decision to participate in the market.
\begin{definition}[Implementable Mechanism]
\label{def:implementable-regulation}
Let $\Pi$ be a regulation mechanism, $P\in\Delta(\mathcal{Z})$ a model provider's type and $G(\Pi,P)\in\{0,1\}$ their ex-ante decision to participate in the market. Then, $\Pi$ is said to \emph{implement} requirement $\mathfrak{R}$ if and only if $G(\Pi,P)=\mathfrak{R}(P)$ for all $P\in\Delta(\mathcal{Z})$.

\end{definition}
We call $\Pi$ that satisfies Definition~\ref{def:implementable-regulation} \emph{implementable}. When model providers are certain about their type $P$, their decision to participate in the market is given by $G(\Pi,P)=\mathbbm{1}[\sup_{\pi\in\Pi}\mathbb{E}_P[\pi(Z)]>C]$. Given the obedience to regulations, $G(\Pi,P)=0$ whenever $\mathfrak{R}(P)=0$ and based on the feasibility of regulations, $G(\Pi,P)=1$ whenever $\mathfrak{R}(P)=1$. Therefore, a mechanism $\Pi$ that satisfies both obedience to regulation and feasibility is also implementable. We now establishes a sufficient and necessary condition for any regulation requirement $\mathfrak{R}$ to be implementable.
\begin{theorem}
\label{theorem:obedience-to-credal}
An implementable regulation mechanism $\Pi$ for a requirement $\mathfrak{R}$ exists if and only if 
\begin{equation*}
   \mathcal{P}_0 := \{P \in \Delta(\mathcal{Z}) \mid \mathfrak{R}(P) = 0\}
\end{equation*}
is a credal set, i.e., a closed, convex set of probability measures. In the special case where the requirement is defined via thresholding rule, i.e., 
\(\mathfrak{R}(P) := \mathbbm{1}[r(P) > \tau]\), an implementable mechanism $\Pi$ exists for any threshold $\tau$ if and only if \(r\) is quasi-convex and lower semi-continuous.\footnote{A functional \(r\) is quasi-convex if all its sublevel sets are convex. Equivalently, for all \(P, Q \in \Delta(\mathcal{Z})\) and \(\lambda \in [0,1]\), $r(\lambda P + (1-\lambda) Q) \le \max\{r(P), r(Q)\}$.}
\end{theorem}
\textbf{Interpretation of credal set in Theorem \ref{theorem:obedience-to-credal}}. From a game-theoretic perspective, AI regulation can be viewed as a game between the regulator (or forecaster) and the model provider (or skeptic). The credal set induced by the specified requirements characterises the conditions under which the regulator cannot be exploited (Dutch booked;~\citealt{definetti1974theory,walley1991}). In classical IP, a forecaster issues predictions alongside gambles, which a skeptic can combine to induce sure loss, thereby exposing the forecaster's internal inconsistency. In our setting, the regulator specifies desired social outcomes through requirements and implements a regulatory mechanism; model providers can then strategically combine evidence distributions to circumvent the regulation, revealing inconsistencies in the regulator's design. In Figure~\ref{fig:geometry_of_regulation}, we show that if $\mathcal{P}_0$ were not convex, a provider with non-compliant models $g$ and $h$ (where $P_{g}, P_{h} \in \mathcal{P}_0$) could simply randomise between them to produce $P_\lambda = \lambda P_{g} + (1-\lambda) P_{h}$ that lies outside $\mathcal{P}_0$, allowing them to obtain a profitable license and bypass regulation without genuinely improving their underlying models. Conversely, if the regulator seeks to prohibit such behaviour, they would end up denying a license to a compliant provider (see the middle of Figure~\ref{fig:geometry_of_regulation}), thereby exposing an internal inconsistency in their position.  

\textbf{Implications of Theorem ~\ref{theorem:obedience-to-credal} for regulators.} Our result enables regulators to understand the market implications of their policies. Regulators often target multiple tasks simultaneously, establishing a checklist of criteria that providers must meet~\citep{lekadir2025future}. However, Theorem 3.5 indicates that even if regulators can guarantee perfect market outcomes for individual tasks, they may fail to do so when these tasks are considered jointly. To see why, consider the non-compliant set for a single task $i$, defined as $\mathcal{P}_0^i=\{P \mid r_i(P)\leq \tau_i\}$. A provider is non-compliant overall if they fail any single task; thus, the global non-compliant set is the union $\mathcal{P}_0=\cup_i\mathcal{P}_0^i$. Since it is a union of convex sets, $\mathcal{P}_0$ is typically non-convex. Many real-world regulatory metrics result in a non-convex $\mathcal{P}_0$, such as group DRO accuracy \citep{sagawa_distributionally_2020} or sub-group fairness \citep{williamson2019fairness}. When faced with non-convexity, regulators can adopt a conservative approach by regulating the convex hull of $\mathcal{P}_0$, i.e., $\text{co}(\mathcal{P}_0)$. While this ensures obedience, it sacrifices feasibility for a provider $Q$ when $Q \in \text{co}(\mathcal{P}_0) \setminus \mathcal{P}_0$. An alternative workaround is to regulate based on surrogate requirements whose $\mathcal{P}_0$ is a credal set—for example, using weighted group accuracy instead of Group-DRO accuracy. On a positive note, Theorem \ref{theorem:obedience-to-credal} permits direct thresholding of any moment of the evidence distribution.



\section{Towards Practical Regulation Mechanisms}

While Theorem~\ref{theorem:obedience-to-credal} provides a prescriptive characterisation for regulation design, it offers limited guidance on constructing implementable $\Pi$ that guarantee perfect market outcomes. To this end, this section provides examples of how to construct an implementable $\Pi$. We now define the set of all licenses that satisfy obedience (Definition~\ref{def:obidience_to_regulation}) as $\Pi^{\textnormal{ob}}_{\mathcal{P}_0}$ and show that:
\begin{lemma}
\label{lemma:all-obidient-or-nothing}
Let $\mathcal{P}_0$ be a credal set, then for an implementable menu $\Pi$ and a license $\pi$ such that $\{\pi\}$ satisfies obedience, $\Pi \cup \{\pi\}$ is also implementable. Consequently, the set of all obedient licenses $\Pi^{\textnormal{ob}}_{\mathcal{P}_0}$ is inherently the largest implementable mechanism.
\end{lemma}
With Lemma~\ref{lemma:all-obidient-or-nothing} regulators can guarantee implementability simply by ensuring obedience. Regulators ask providers to commit to their model's compliance by proposing a license  $\pi\in\Pi^{\textnormal{ob}}_{\mathcal{P}_0}$. In doing so, the burden of ensuring feasibility is shifted on to the providers. To operationalise this, regulators must be able to verify whether a proposed $\pi$ belongs to $\Pi^{\textnormal{ob}}_{\mathcal{P}_0}$. To this end, we characterise $\Pi^{\textnormal{ob}}_{\mathcal{P}_0}$ below:
\begin{proposition}
\label{prop:characterisation-and-invariance-of-obedience}
     Let $\mathfrak{G}_{\leq 0,\mathcal{P}_0}$ be the set of marginally undesirable gambles w.r.t $\mathcal{P}_0$ and $\mathfrak{D}_{\geq 0}^R$ be the set of all desirable gambles with max payout of $R$, then $\Pi^{\textnormal{ob}}_{\mathcal{P}_0}= \{\mathfrak{G}_{\leq0,\mathcal{P}_0}+C\}\cap \mathfrak{D}_{\geq 0}^R$ where the set $\{\mathfrak{G}_{\leq0,\mathcal{P}_0}+C\}:= \{g+c \mid g\in\mathfrak{G}_{\leq 0,\mathcal{P}_0}\}$. Given such a characterisation, we can write $\Pi^\textnormal{ob}_{\mathcal{P}_0}$ as 
    \begin{equation}
    \label{eq:characterise-all-obedient-licenses}
        \Pi^{\textnormal{ob}}_{\mathcal{P}_0}:=\left\{\pi:\mathcal{Z}\rightarrow[0,R] \,\middle|\, \sup_{P\in\mathcal{P}_0}\mathbb{E}_P[\pi(z)]\leq C\right\}.
    \end{equation}
\end{proposition}
Regulators can verify if the proposed $\pi$ belongs to $\Pi^{\text{ob}}_{\mathcal{P}_0}$ via Equation~\ref{eq:characterise-all-obedient-licenses}. Furthermore, $\Pi^{\text{ob}}_{\mathcal{P}_0}$ as defined in Equation~\ref{eq:characterise-all-obedient-licenses} satisfies Definition~\ref{def:obidience_to_regulation} and is closed, convex and invariant up to $\text{co}(\mathcal{P}_0)$; see Appendix~\ref{subsection:properties_of_obedient_regulations} for the proof. Characterising $\Pi^{\text{ob}}_{\mathcal{P}_0}$ via marginally undesirable gambles aligns with our interpretation of credal sets and the regulator in Theorem~\ref{theorem:obedience-to-credal}. By offering $\Pi^{\text{ob}}_{\mathcal{P}_0}$ as a menu of licenses, the regulator effectively acts as a forecaster, issuing undesirable gambles with respect to $\mathcal{P}_0$.

\textbf{Connection to hypothesis testing:} While $\Pi^{\text{ob}}_{\mathcal{P}_0}$ is implementable and is characterised by Equation~\ref{eq:characterise-all-obedient-licenses}, it is an infinite set. Thus, in practice, proposing a feasible $\pi \in \Pi^{\text{ob}}_{\mathcal{P}_0}$ is challenging for model providers. To address this issue, we establish a link between implementability and hypothesis testing. For any type $Q\in\Delta(\mathcal{Z})\setminus \mathcal{P}_0$, we can characterise the feasible $\pi \in \Pi^{\text{ob}}_{\mathcal{P}_0}$ as a continous function in $\mathcal{C}(\mathcal{Z})$ that satisfies the following feasibility, obedience, and market cap constraints:
\begin{equation}   
\mathbb{E}_{Q}[\pi(z)] > C, \quad \sup_{P\in\mathcal{P}_0}\mathbb{E}_{P}[\pi(z)]\leq C, \quad 0\leq \pi(z)\leq R.
\label{eq:optimisation-all-or-nothing-bets}
\end{equation}
By absorbing the market cap constraint into the definition of licenses and dividing by $R$, we can express the license in terms of a statistical test $\phi:\mathcal{Z}\to[0,1]$ as $\pi = R\cdot \phi$ such that the test $\phi$ satisfies unbiasedness and false positive control, i.e.,
\begin{align} 
\mathbb{E}_{Q}\left[\phi(z)\right]>C/R, \quad \sup_{P\in\mathcal{P}_0}\mathbb{E}_{P}\left[\phi(z)\right]\leq C/R .
\label{eq:credal-testing}
\end{align}
In other words, when $\mathcal{P}_0$ is a singleton $\{P\}$, a test that satisfies Equation~\ref{eq:credal-testing} is essentially an unbiased hypothesis test between $Q$ (alternate) and $P$ (null) with false positive rate of $\alpha=C/R$. A test $\phi$ is considered unbiased when $\alpha<1-\beta$ where $\beta$ is the type II error~\citep{lehmann2005testing}. Let $\phi_Q$ be an unbiased test between $Q$ and $P$. Then, an implementable mechanism can be explicitly constructed as $\Pi=\{R\cdot\phi_Q \mid \forall Q \neq P\}$. Since regulators construct $\pi$ with tests with the null hypothesis that the provider is non compliant, the mechanism build with these collection of hypothesis tests places the burden of proof onto model providers. Similar connections between contract theory and hypothesis testing with $\mathcal{P}_0=\{P\}$ have been drawn before in the context of moral hazard~\citep{saig2023delegated} and adverse selection~\citep{bates2022principal}. 

However, when $\mathcal{P}_0$ is not a singleton, the tests become an unbiased test between a simple alternate $Q$ and a composite null $\mathcal{P}_0$. In this case, Theorem~\ref{theorem:obedience-to-credal} implies that we can construct an implementable $\Pi$ if and only if $\mathcal{P}_0$ is a credal set. Thus, we can leverage the following test, which we denote as a credal test to build an implementable $\Pi$ for composite $\mathcal{P}_0$:
\begin{align}
    H_0:P\in\mathcal{P}_0 \quad\text{ vs }\quad H_1:P=Q \quad \text{ where } \mathcal{P}_0 \text{ is a credal set.}
\label{eq:credal test}
\end{align}
\begin{proposition}
\label{prop:credal-implementable-menu}
    Let $\mathcal{P}_0$ be a credal set and $\bm{\Phi}=\{\phi:\mathcal{Z}\rightarrow[0,1]\}$ be a set of credal tests. Suppose that, for every $\phi\in\bm{\Phi}$, $\alpha=C/R$ and for every $Q \notin \mathcal{P}_0$, there exists at least one unbiased test in $\phi\in\bm{\Phi}$. Then, $\Pi=\{R\cdot\phi \mid \phi\in\bm{\Phi}\}$ is implementable. 
\end{proposition}
 The rich literature in robust hypothesis testing~\citep{huber1965robust, huber1973minimax,augustin1998optimale,levy2008robust,schwaferts2019imprecise} can be used to build an implementable $\Pi$ in practice. However, unlike the singleton case, the Huber-Strassen test requires $\mathcal{P}_0$ to be 2-monotone~\citep{sundberg1992characterizations} which could restrict its applications.  

\textbf{Connection to sequential hypothesis testing:} While static hypothesis tests can be used to build an implementable $\Pi$, regulators often observe evidence over time, such as through continuous API monitoring. To this end, we show that recent advances in sequential hypothesis testing (See \citet{ramdas2024hypothesis}) can also be used to construct an implementable $\Pi$ as follows:
\begin{proposition}
\label{propositon:sequential-implementable-menu}
Let $\mathcal{P}_0$ be a compact credal set and $R\gtrsim 2.73C$, then $\Pi=\{C\cdot\phi_{\tilde{P}}\mid\forall \tilde{P}\notin\mathcal{P}_0\}$ is implementable. Where for an agent with type $Q$,
\begin{align}
    &\phi_{\tilde{P}}(z):= \min\Bigg\{\frac{\tilde{P}(z)}{P^*(z)},\frac{R}{C}\Bigg\}, \quad\text{and}\quad P^*=\argmin_{P\in\mathcal{P}_0} \;\text{KL}(Q \| P). 
\label{eq:sequential-licenses}
\end{align}
\end{proposition}

The test$\phi_{\tilde{P}}$ effectively computes the likelihood ratio between agent's declared type $\tilde{P}$ and the true type $Q$ by searching for $P^*\in\mathcal{P}_0$. $P^*$ is the most similar distribution to $Q$ in $\mathcal{P}_0$ in terms of KL-divergence, also known as the Reverse Information Projection \citep{li1999estimation,csiszar2003information}. The likelihood ratio is truncated due to $\|\pi\|_{\infty}\leq R$. The regulator can now assign license based sequence of accumulated evidence $\{z_i\}_{i=1}^n$. to a provider with type $Q$ as $\pi_Q(z_{1:n})\;=\;\min\{C \prod_{i=1}^n \frac{Q(z_i)}{P^*(z_i)},\; R\}$.

So far, we assumed the regulator has an explicit form of $\mathcal{P}_0$. In practice, $\mathcal{P}_0$ is often defined implicitly (e.g., via fairness or risk constraints), making the likelihood ratio in Equation~\ref{eq:sequential-licenses} intractable. To enable the regulators to build an implementable $\Pi$ without explicit access to $\mathcal{P}_0$, we resort to the testing-by-betting framework \citep{shafer2021testing, grunwald2024safe} under the following assumption:
\begin{assumption}
    The requirement can be expressed as $\mathfrak{R}(P)=\mathbbm{1}[r(P)>\tau]$ for some $r(P):=\mathbb{E}_{P}[h(z)]$ such that  $h:\mathcal{Z}\rightarrow\mathbb{R}$ is the betting score or an unbiased estimator of $r$.
\label{assumption:linearisable_requirements}
\end{assumption}
Then, let us consider that regulator obtains finitely many i.i.d samples $\{z\}_{i=1}^n$ from the evidence generating distribution. Under Assumption \ref{assumption:linearisable_requirements}, the regulator can offer a mechanism $\Pi_n:=\{\pi_n[\lambda]\}_{\lambda}$ where $\pi_n[\lambda]:\mathcal{Z}^n\rightarrow[0,R]$ is the license based on $n$ samples:
\begin{equation}
    \pi_n[\lambda](z_{1:n}):=\min\Bigg\{C\prod_{i=1}^n\Big(1+\lambda_i(h(z_i)-\tau)\Big),R\Bigg\}
\label{eq:martingale-test}
\end{equation}
where $\lambda = (\lambda_n)_{n \geq 1}$ parametrises $\pi$ and allows providers to select a feasible $\pi_n\in\Pi_n$, subject to the constraint that $\lambda_n \in [0, B_n]$ for some $B_n\in\mathbb{R}$ such that $1 + \lambda_n (h(z_n)-\tau) \geq 0$ almost surely.
\begin{proposition}
    In cases where Assumption~\ref{assumption:linearisable_requirements} holds, there exists an $N\in\mathbb{N}$ such that for all  $n\geq N$, $\Pi_{n}:=\{\pi_n[\lambda]\}_{\lambda}$ is implementable.
\label{proposition:matringale-implementable}
\end{proposition}
The Equation \ref{eq:martingale-test} used to construct a $\pi\in\Pi$ is commonly referred to as a test-martingale~\cite{ramdas2024hypothesis} in the testing by betting literature which implicitly tests against a credal set (See Appendix~\ref{app:martingale_background}). Test-martingales have been applied to monitor risk ~\citep{waudby2024estimating, timans2025continuous}, fairness~\citep{chugg2023auditing} and differential privacy~\citep{gonzalez2025sequentially} and they can be transformed into implementable mechanisms for cases where $\mathcal{P}_0$ is a credal set. 

\section{Experiments}

We empirically validate our theoretical contributions through three experiments: (1) \textbf{Strategic Gaming}: We demonstrate that regulators using a non-convex $\mathcal{P}_0$ are vulnerable to arbitrage by strategic agents (Fig~\ref{experiments:fig1a}); (2) \textbf{Example of }$\bm{\Pi}$ \textbf{using Sequential Hypothesis Testing}: We design licenses $\pi\in\Pi$ using Equation~\ref{eq:sequential-licenses} for regulators seeking to control the use of spurious features in classification under sequentially observed data (Figs~\ref{experiments:fig1b} \& \ref{experiments:fig1c}); and (3) \textbf{Regulation with Implicit $\mathcal{P}_0$}: We demonstrate a fairness regulation framework in which $\mathcal{P}_0$ is specified implicitly, eliminating the need for an explicit representation (Fig~\ref{experiments:fig1d}). 

\subsection{Datasets and Experiment Setup} 
\vspace{-0.5em}
In Figure~\ref{experiments:fig1a}, we define an outcome space with three prohibited distributions: $P_1=[0.35, 0.35, 0.3]$, $P_2=[0.35, 0.3, 0.35]$, and $P_3=[0.3, 0.35, 0.35]$. A model provider then samples uniformly from $\{P_1, P_2, P_3\}$, effectively generating the evidence distribution $Q=\sum_{i=1}^3 \frac{1}{3}P_i$. The strategic non-compliant provider attempts to bypass a ``naive'' regulator that grants licenses by testing against the discrete set $\{P_1, P_2, P_3\}$ using the generalised likelihood ratio $\min\big\{\frac{dQ(z)}{\max_i P_i(z)},R\big\}$, whereas the credal regulator grants licenses by testing against the entire credal set.

We use the Waterbirds dataset~\citep{sagawa_distributionally_2020}, a popular benchmark for learning under spurious correlations, where task is to classify birds as \emph{Landbirds} or \emph{Waterbirds}. The training data is heavily biased: $95\%$ of waterbirds appear on water backgrounds (spurious correlation). We compare two agents: (1) \textbf{A Non-compliant Agent}: An ERM model trained which often relies on spurious features to make prediction.
(2) \textbf{A Compliant Agent}: A model trained via Group-DRO, which is less susceptible to the spurious features. We use ResNet-50~\citep{he2016deep} to train both models. The regulator explicitly defines $\mathcal{P}_0$ via the convex hull of an ERM-trained model mixed with a random predictor. Essentially, $\mathcal{P}_0$ represents the mixture of distributions which rely on the spurious features, background information in the case of ERM and random noise in the case of random predictor.

In Figure~\ref{experiments:fig1d}, we consider a regulator who enforces a demographic parity: $|\mathbb{E}[Y \mid A=1] - \mathbb{E}[Y \mid A=0]| < \tau$ for prediction $Y\in\{0,1\}$, subgroups $A\in\{0,1\}$, and $\tau=0.6$.  We simulate providers with varying true fairness gaps $\Gamma \in \{0.4, 0.6\}$ by setting the prediction rates for the subgroups to fixed Bernoulli parameters as $Y_0 = \text{Bernoulli}(0.1)$ and $Y_1 = \text{Bernoulli}(\Gamma+0.1)$. Here, the regulator does not maintain a representation of ``unfair distributions''. Instead, they offer a license based on the statistic
$\pi_n = \prod_{t=1}^n (1 + \lambda_t (\tau-|Y_0-Y_1|))$ where $\lambda_t$ indicates an adaptive betting strategy~\citep{shekhar2023reducing}. See Appendix~\ref{appendix:experiment-details-fairness} for further experimental details.
\vspace{-0.5em}
\begin{figure*}[t]
    \begin{subfigure}{0.22\textwidth}
        \includegraphics[width=\linewidth]{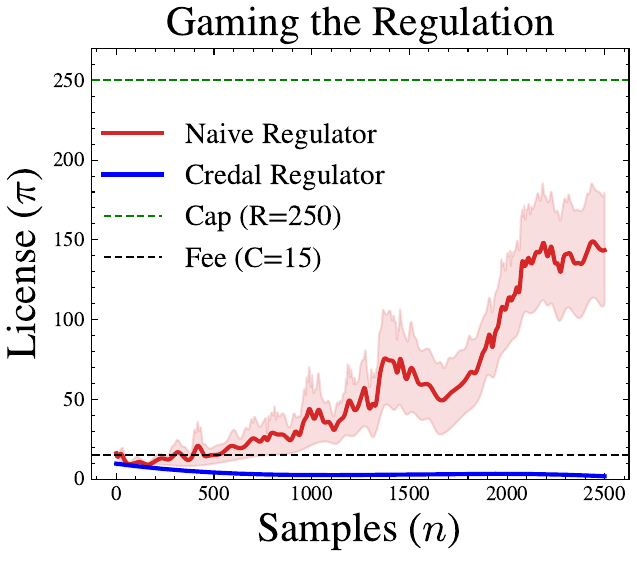}
        \caption{}
        \label{experiments:fig1a}
    \end{subfigure}
    \begin{subfigure}{0.22\textwidth}
        \includegraphics[width=\linewidth]{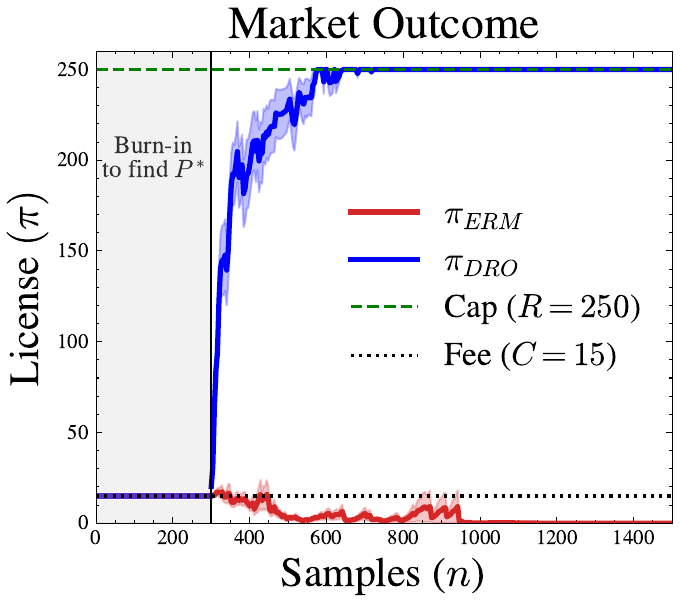}
        \caption{}
        \label{experiments:fig1b}
    \end{subfigure}
    \begin{subfigure}{0.22\textwidth}
        \includegraphics[width=\linewidth]{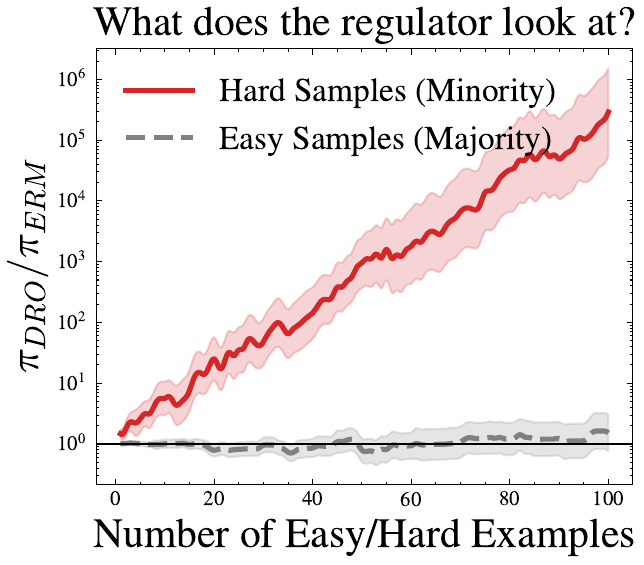}
        \caption{}
        \label{experiments:fig1c}
    \end{subfigure}
    \begin{subfigure}{0.325\textwidth}
        \includegraphics[width=\linewidth]{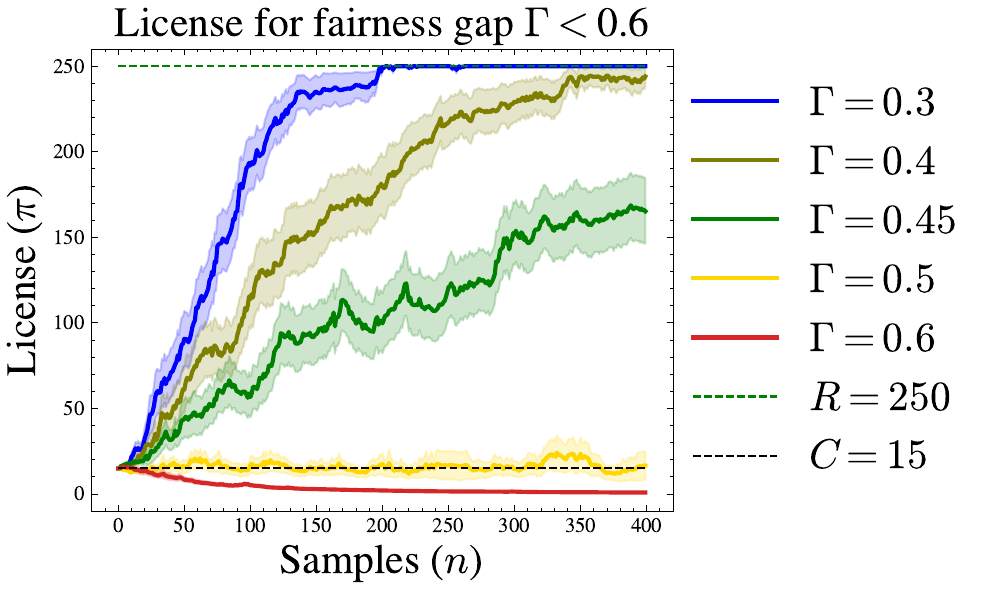}
        \caption{}
        \label{experiments:fig1d}
    \end{subfigure}
    \caption{(\subref{experiments:fig1a}) A regulator with a non-convex $\mathcal{P}_0$ can be exploited by strategic providers through mixtures of evidence generated by poor models. (\subref{experiments:fig1b}) Accumulated revenue generated by \(\pi\) for ERM (non-compliant) and Group-DRO (compliant) agents on Waterbirds as the sample size \(n\) increases. (\subref{experiments:fig1c}) The ratio \(\pi_{\mathrm{DRO}} / \pi_{\mathrm{ERM}}\) computed on 100 random test samples, separated into easy (majority) and hard (minority, counter-spurious) examples, showing that the Group-DRO agent receives a more favorable license due to its stronger performance on hard examples. (\subref{experiments:fig1d}) Practical fairness regulations based on implicit credal set. Results are averaged over 30 runs; shaded regions denote standard error.}
    \label{fig:waterbirds}
\end{figure*}
\subsection{Insights from Experiments}
\vspace{-0.5 em}
\textbf{Credal regulators limit strategic behaviour.} Figure~\ref{experiments:fig1a} highlights the vulnerability of non-convex regulation. The naive regulator (red) grants a license to the strategic provider because the mixture distribution $Q$ is statistically distinct from every individually prohibited evidence distribution $P_i$. In contrast, the credal regulator (green) correctly identifies that $Q$ lies within the convex hull of prohibited distributions and the provider self excludes. This demonstrates that, for a regulation to be robust against strategic behaviour, its set of prohibited distributions must be a credal set. 

\textbf{Regulation with explicit $\mathcal{P}_0$.} While thresholding directly on worst-case group accuracy would lead to a non-convex $\mathcal{P}_0$, regulators could explicitly construct a credal set $\mathcal{P}_0$ which works in practice. Figures~\ref{experiments:fig1b} and \ref{experiments:fig1c} demonstrate this for regulating the use of spurious features. In Figure~\ref{experiments:fig1b}, after a 300-sample burn-in to find $P^*$, the compliant agent's license grows to the cap $R=250$, while the non-compliant agent fails to obtain a license that can recover its fee $C=15$. Figure~\ref{experiments:fig1c} shows that the larger license value of the compliant agent is driven  by ``Hard Examples" (e.g., Waterbirds on Land), as on ``Easy Examples'' both agents agree with the regulator's baseline ($\pi_{\text{DRO}}/\pi_{\text{ERM}} \approx 1$).    

\textbf{Regulation with implicit $\mathcal{P}_0$.} 
Figure~\ref{experiments:fig1d} shows that regulators need not maintain an explicit credal set to issue licenses. Regulators can offer model providers the opportunity to bet on their model's fairness by selecting $\lambda$, thereby enabling implicit testing against the credal set of all non-compliant distributions. Since $|Y_0-Y_1|$ is not an unbiased estimator of the fairness gap $\Gamma$, borderline non-compliant providers ($\Gamma=0.5$) may self-exclude. Nevertheless, the surrogate estimator $|Y_0-Y_1|$ preserves obedience (see Appendix~\ref{appendix:experiment-details-fairness} for a proof). In practice, the regulation also does not significantly compromise feasibility, as compliant agents begin to participate from $\Gamma=0.45$ onward. 
\section{Related Work}

In this section, we contrast our work with existing research on principal-agent hypothesis testing, imprecise probability, testing-by-betting, and AI governance.

\textbf{Principal-agent problem and hypothesis testing.} Several recent works study relationship between the principal-agent problem and hypothesis testing~\citep{bates2022principal, bates2023incentive, min2023screening, hossain2025strategic}. In particular, \citet{bates2022principal, bates2023incentive, min2023screening} study adverse selection while focusing on moral hazard. \citet{hossain2025strategic} formulate strategic testing as a Bayesian games, whereas \citet{gauthier2026betting} extend it to testing for equilibrium. Conceptually, our work is most closely related to \citet{bates2022principal} as we also formulate core challenge in AI regulation as information asymmetry under uncertainty. However, our focus is on characterising perfect market outcomes, rather than mitigating strategic behaviour. See Appendix~\ref{appendix:hypothesis-testing} for further discussion.

\textbf{IP and testing by betting.} Our results characterises implementable regulation mechanisms, using tools from Imprecise Probability (IP)~\citep{walley1991,augustin_introduction_2014} and theory of desirability~\citep{de2012exchangeability,de2023theory}. IP offers rich literature on credal testing \citep{huber1973minimax,chau2024credal,jurgens2025calibration,chugg2026admissibility} which in light of our results can have potential applications in AI regulations. Closely related to IP testing literature is the testing-by-betting literature~\citep{shafer2021testing,vovk_game-theoretic_nodate,ramdas2023game,grunwald2024safe}.  Testing-by-betting has become the backbone of methods that perform auditing of ML models~\citep{shekhar2023reducing,xu2024online}, LLM providers~\citep{velasco2025auditing} and AI agents~\citep{sadhuka2025valuator} for different applications. Our results show that testing by betting methods can be used to build implementable regulation mechanisms. 

\textbf{AI governance and regulation.} The rapid AI adoption has significant societal consequences and has led to ongoing debates on AI governance \citep{dafoe2018ai,taddeo2018ai}. This debate has lead to question regulations from an ethical~\citep{jobin2019global,hagendorff2020ethics,huang2022overview}, policy~\citep{diakopoulos2016accountability,o2017weapons}, socio-cultural~\citep{awad2018moral,vesnic2020societal}, and political~\citep{pavel2023ai,schmid2025arms} perspectives. Our work contributes to this literature by providing insights and surfacing technical challenges in operationalising policymakers’ normative goals~\citep{kroll2015accountable,raji2020closing, kolt2024responsible, lekadir2025future}. We review these challenges further in Appendix~\ref{appendix-non-technical-challenges-for-regulation} and discuss the broader impact of our work in Appendix~\ref{app:broader-impact}.
\section{Discussion}
Many proposals have argued to shift the burden of proof onto providers to operationalise black-box regulations. For example, \citet{buhl2024safety} argued for providers to construct structured safety cases as evidence that their systems meet the regulator's standards. \citet{hadfield2023regulatory} and \cite{bova2023both} proposed establishing regulatory markets in which providers must purchase compliance services. \citet{cen2024transparency} deliberated burden of proof while designing hypothesis tests for AI regulation such that the tests can fulfil corresponding legal requirements. The shared logic across all these proposals is that providers know more about their models than regulators and should bear the risk of demonstrating compliance. Our main result, Theorem~\ref{theorem:obedience-to-credal}, characterises the cases in which these proposals could be operationalised. The main challenge moving forward lies in managing the trade-offs introduced by the impossibility of designing PMO-achieving mechanisms for non-credal $\mathcal{P}_0$. This involves both open policy questions and technical challenges concerning practical ways to circumvent the impossibility result. On the positive side, an immediate consequence of Theorem~\ref{theorem:obedience-to-credal} is that any non-credal $\mathcal{P}_0$ can be \emph{conservatively} regulated by its closed convex hull $\overline{\mathrm{co}}(\mathcal{P}_0)$. However, when Theorem~\ref{theorem:obedience-to-credal} permits PMO for a regulation, the regulators can design implementable mechanisms with great flexibility using collection of hypothesis tests for both static and sequential settings. The implementable mechanisms guarantee PMO while allowing regulators to bypass noisy, sample-based verification of compliance by shifting burden of proof on to providers.


\newpage
\addcontentsline{toc}{section}{Appendix} 
\part{Appendix} 
\parttoc 
\newpage
\section{Proof of Theorem ~\ref{theorem:obedience-to-credal}}
In order to prove our main theorem we define some additional auxilary concepts below. In general the main idea of the proof is that we look at the set of distributions that are regulated by any regulation mechanism $\Pi$ and then ask under what conditions are the regulated distributions exactly the distributions that we want to regulate i.e. $\mathcal{P}_0=\left\{P\in\Delta(\mathcal{Z}\middle|\mathfrak{R}(P)=0\right\}$. To this end we define the set of regulated distributions. 
\begin{definition}
We define a set of regulated distributions $\mathcal{P}_{reg}(\Pi)$ as the set of distributions that are regulated by a mechanism $\Pi\subseteq \mathcal{C}(\mathcal{Z})$. Formally, this means that
\begin{align*}
    P\in\mathcal{P}_{reg}(\Pi)\quad \text{iff} \quad \sup_{\pi\in\Pi} \mathbb{E}_{P}[\pi(Z)]\leq C
\end{align*}
\end{definition}
Careful readers can note that the condition in the definition of regulated distributions is similar to Definition~\ref{def:obidience_to_regulation}. However, in general $\mathcal{P}_{reg}(\Pi)$ may not be similar to $\mathcal{P}_0$. Also, $\mathcal{P}_{reg}(\Pi)$ is convex and closed for any $\Pi\subseteq \mathcal{C}(\mathcal{Z})$, i.e. $\mathcal{P}_{reg}(\Pi)$ is always a credal set. Verifying this is actually quite straightforward and we include this for completeness of the proof. Let us consider an arbitrary $\Pi$, then 
\paragraph{Claim 1: (\mbox{$\mathcal{P}_{reg}(\Pi)$} is a convex set)}
Let us assume that there exist two arbitrary distributions such that $P_1,P_2\in\mathcal{P}_{reg}(\Pi)$. And for $\alpha\in[0,1]$ there exists a linear combination $P_{\alpha}:=\alpha P_1+(1-\alpha)P_2$ of $P_1$ and $P_2$. Since $P_1,P_2\in\mathcal{P}_{reg}(\Pi)$, 
\begin{align*}
    \sup_{\pi\in\Pi}\mathbb{E}_{Z \sim P_1}[\pi(Z)]&\leq C \quad\text{ and } \quad \sup_{\pi\in\Pi}\mathbb{E}_{Z \sim P_2}[\pi(Z)]\leq C
\end{align*}
Then, 
\begin{align*}
    \forall \alpha \in [0,1] \quad \sup_{\pi\in\Pi} \mathbb{E}_{Z\sim P_\alpha}[\pi(Z)]&\leq \alpha \sup_{\pi\in\Pi}\mathbb{E}_{Z \sim P_1}[\pi(Z)] + (1-\alpha)\sup_{\pi\in\Pi}\mathbb{E}_{Z \sim P_2}[\pi(Z)] \quad \tag*{(\mbox{$P\rightarrow \sup_{\pi\in\Pi}\mathbb{E}_{P}[\pi(Z)]$} is convex in \mbox{$P$})}\\
    &\leq \alpha C + (1-\alpha)C  \tag*{(\mbox{$P_1,P_2\in\mathcal{P}_{reg}(\Pi)$})}\\
    &= C
\end{align*}
Since $\sup_{\pi\in\Pi} \mathbb{E}_{Z\sim P_\alpha}[\pi(Z)]\leq C$, $P_{\alpha}\in\mathcal{P}_{reg}(\Pi)$. This holds for all $\alpha\in[0,1]$ for arbitrary choices of $P_1$ and $P_2$ in $\mathcal{P}_{reg}(\Pi)$. Therefore, $\mathcal{P}_{reg}(\Pi)$ is a convex set for any $\Pi$. 
\paragraph{Claim 2: (\mbox{$\mathcal{P}_{reg}(\Pi)$} is a closed set)} We consider weak-* topology on our space of probability measures $\Delta(\mathcal{Z})$. We know that $\Pi\subseteq \mathcal{C}(\mathcal{Z})$ is a set of bounded, non-negative, continuous functions, i.e. $0\leq\pi(z)\leq R$ for all $z\in\mathcal{Z}$ for all $\pi\in\Pi$. We also know that $C \in \mathbb{R}_{\geq 0}$ is a constant. And we defined the set $\mathcal{P}_{reg} \subseteq \Delta(\mathcal{Z})$ as:

$$\mathcal{P}_{reg}(\Pi) = \left\{ P \in \Delta(\mathcal{Z}) \;\middle|\; \sup_{\pi \in \Pi} \mathbb{E}_P[\pi(Z)] \leq C \right\}$$

Therefore, the condition $\sup_{\pi \in \Pi} \mathbb{E}_P[\pi(Z)] \leq C$ is satisfied if and only if $\mathbb{E}_P[\pi(Z)] \leq C$ holds for every individual $\pi \in \Pi$. We can rewrite the set $\mathcal{P}_{reg}(\Pi)$ as an intersection of sets defined by single constraints:

\begin{equation*}
    \mathcal{P}_{reg}(\Pi) = \bigcap_{\pi \in \Pi} \left\{ P \in \Delta(\mathcal{Z}) \;\middle|\; \mathbb{E}_P[\pi(Z)] \leq C \right\}
\end{equation*}
Let us define the constituent sets as $A_\pi$:
\begin{equation*}
A_\pi := \left\{ P \in \Delta(\mathcal{Z}) \;\middle|\; \mathbb{E}_P[\pi(Z)] \leq C \right\}    
\end{equation*}
Now, $\mathcal{P}_{reg}(\Pi) = \bigcap_{\pi \in \Pi} A_\pi$, and since arbitrary intersection of closed sets is closed, in order to prove that $\mathcal{P}_{reg}(\Pi)$ is closed we need to show that all $A_\pi$ are closed sets. Let's choose an arbitrary $\pi \in \Pi$ and since $A_{\pi}$ is a half space in the space of distributions $\Delta(\mathcal{Z})$, equipped with weak-* topology, we need to show half spaces of probability measures are closed in the weak-* topology. To this end we state the notion of convergence of a sequence of distributions $\{P_n\}_{n=1}^\infty$ to its limit $P$ in the weak-* topology. 
\begin{definition}(Convergence of distributions in weak-* topology) We say that a sequence of distributions $\{P_n\}_{n=1}^\infty$ converges to $P$ in the weak-* topology if 
\begin{equation*}
    \mathbb{E}_{P_n}[f]\rightarrow \mathbb{E}_{P}[f]\quad \forall\text{ }f\in \mathcal{C}(\mathcal{Z})
\end{equation*}
\label{definition:weak-star-convergence}
\end{definition}
Within our arbitrary $A_\pi$, let us now consider a sequence of $\{P_n\}_{n=1}^\infty\in A_\pi$ such that $\{P_n\}_{n=1}^\infty\to P$. This implies convergence of the expectations $E_{P_n}[\pi] \to E_P[\pi]$ as $\pi\in\Pi\subseteq \mathcal{C}(\mathcal{Z})$ according to Definition~\ref{definition:weak-star-convergence}. Since $\{P_n\}_{n=1}^\infty \in A_\pi$, we know that $E_{P_n}[\pi] \leq C$ for all $n$. Since the limit of a sequence of numbers less than or equal to $C$ must also be less than or equal to $C$ that is:
\begin{align*}
   E_P[\pi] = \lim_{n \to \infty} E_{P_n}[\pi] \leq C 
\end{align*}
Hence, $P \in A_\pi$ and $A_\pi$ is closed. Therefore $\mathcal{P}_{reg}(\Pi)$ is closed.
\begin{lemma}
    $\mathcal{P}_0=\mathcal{P}_{reg}(\Pi)$ if an only if $\Pi$ is implementable.
\label{lemma:equivalance-of-preg-and-pie-impl}
\end{lemma}
\begin{proof}
$(\Rightarrow)$

We prove this with contradiction. Let us assume that $\mathcal{P}_0=\mathcal{P}_{reg}(\Pi)$ and $\Pi$ is not implementable. From the definition of $\mathcal{P}_{reg}(\Pi)$ we know that 
$$\mathcal{P}_{reg}(\Pi) = \left\{ P \in \Delta(\mathcal{Z}) \;\middle|\; \sup_{\pi \in \Pi} \mathbb{E}_P[\pi(Z)] \leq C \right\}$$
Since $\mathcal{P}_0=\mathcal{P}_{reg}(\Pi)$, this implies that $\Pi$ is obedient to the regulation. However, $\Pi$ is implementable which means that it must not be feasible. As a consequence of infeasibility, there must exist a $P\in\Delta(\mathcal{Z})\setminus\mathcal{P}_0$ such that 
\begin{align}
    \forall \pi\in\Pi \quad \mathbb{E}_{P}[\pi(Z)]\leq C \quad \implies \sup_{\pi\in\Pi} \mathbb{E}_{P}[\pi(Z)]\leq C
\end{align}
Hence, $P\in\mathcal{P}_{reg}(\Pi)$ by definition of $\mathcal{P}_{reg}(\Pi)$. Since we assume that $\mathcal{P}_0=\mathcal{P}_{reg}$, then $P\in\mathcal{P}_0$. This leads to a contradiction as $(\Delta(\mathcal{Z})\setminus\mathcal{P}_0)\cap \mathcal{P}_0=\emptyset$.

$(\Leftarrow)$

We prove this by proving the contrapostive. Let us assume that $\mathcal{P}_0\neq \mathcal{P}_{reg}(\Pi)$. Which results in the following two cases. 
\paragraph{Case 1: \mbox{$\exists$} \mbox{$P$} such that \mbox{$P\in\mathcal{P}_0$} but \mbox{$P\notin\mathcal{P}_{reg}(\Pi)$}} This implies that there exists a $P\in\mathcal{P}_0$ such that $\sup_{\pi\in\Pi}\mathbb{E}_{P}[\pi(Z)]\not\leq C$ as $P\notin\mathcal{P}_{reg}(\Pi)$. Therefore, $\Pi$ is not obedient. Hence, $\Pi$ is not implementable. 
\paragraph{Case 2: \mbox{$\exists$} \mbox{$P$} such that \mbox{$P\in\mathcal{P}_{reg}(\Pi)$} but \mbox{$P\notin\mathcal{P}_{0}$}}. This implies that $P\in\Delta(\mathcal{Z})\setminus\mathcal{P}_0$. However, since $P\in\mathcal{P}_{reg}(\Pi)$, we know that $\sup_{\pi\in\Pi}\mathbb{E}_{P}[\pi(Z)]\leq C$. This reaches a contradiction and $\Pi$ can not be feasible. Hence, $\Pi$ is not implementable.  
\end{proof}
\subsection*{Completing the Proof of Theorem \ref{theorem:obedience-to-credal}}
\textbf{First Part: There exists and implementable $\Pi$ iff $\mathcal{P}_0$ is a Credal Set.}
\\~\\
We now show that the existence of an implementable regulation mechanism $\Pi$ implies that $\mathcal{P}_0$ is Credal. Assume there exists an implementable $\Pi$. By Lemma \ref{lemma:equivalance-of-preg-and-pie-impl}, this implies $\mathcal{P}_0 = \mathcal{P}_{reg}(\Pi)$. By Claims 1 and 2, $\mathcal{P}_{reg}(\Pi)$ is always convex and closed. Therefore, $\mathcal{P}_0$ must be a credal set.

We now show the opposite direction. Assume $\mathcal{P}_0$ is a credal set then there exists an implementable regulation mechanism $\Pi$. We rely on the Hahn-Banach Separation Theorem (See Chapter 3 \citep{rudin1991functional}), which states that any closed convex set is the intersection of all closed half-spaces containing it. Therefore,
for every $Q \notin \mathcal{P}_0$, there exists a continuous linear functional $h_Q \in \mathcal{C}(\mathcal{Z})$ such that:
$$ \sup_{P \in \mathcal{P}_0} \mathbb{E}_P[h_Q(Z)] < \mathbb{E}_Q[h_Q(Z)] $$
To construct a mechanism from the set $\{h_Q\mid Q\notin\mathcal{P}_0\}$, let us define these separated expectations as, $\mu_0 := \sup_{P \in \mathcal{P}_0} \mathbb{E}_P[h_Q(Z)]$ and $\mu_Q := \mathbb{E}_Q[h_Q(Z)]$ where $\mu_0 < \mu_Q$. Because $h_Q \in C(\mathcal{Z})$ is bounded, its infimum and supremum over the domain $\mathcal{Z}$ exist as finite real numbers. Let $h_{min} = \inf_{z \in \mathcal{Z}} h_Q(z)$ and $h_{max} = \sup_{z \in \mathcal{Z}} h_Q(z)$. Then $h_{min}\leq \mu_0<\mu_Q\leq h_{max}$. We can construct our valid license $\pi_Q(z)$ by applying an affine transformation to $h(z)$.i.e. $\pi_Q(z) = \alpha h(z) + \beta$ where $\alpha > 0$ is a scaling factor and $\beta \in \mathbb{R}$ is a translation shift. Then $\Pi = \{ \pi_Q \mid Q \notin \mathcal{P}_0 \}$ can be shown to be both obedient and feasible. To exactly satisfy obedience, we choose the transformation such that:
\begin{align*}
    \sup_{P \in \mathcal{P}_0} \mathbb{E}_P[\pi_Q(Z)] = \alpha \mu_0 + \beta = C
\end{align*}
Since we can do this for every $\pi_Q$, $\Pi$ satisfies obedience by construction. Now, solving for $\beta$ yields $\beta = C - \alpha \mu_0$. We can verify feasibility as follows: 
\begin{align*}
    \mathbb{E}_Q[\pi_Q(Z)] &= \alpha\mathbb{E}_Q[h_Q(Z)]  + \beta\\
    &= \alpha \mu_Q + C - \alpha \mu_0\\ 
    &= C + \alpha(\mu_Q - \mu_0) > C
\end{align*}. 

Additionally, we must ensure the transformed license conforms to $0 \le \pi_Q(z) \le R$ for all $z \in \mathcal{Z}$:
\begin{align*}0 \le \alpha h_Q(z) + C - \alpha \mu_0 \le R.
\end{align*}
Therefore, $\alpha(\mu_0 - h_{min}) \le C$ and $\alpha(h_{max} - \mu_0) \le R - C$. If $\mu_0 > h_{min}$ then $\alpha = \frac{1}{2} \min \left( \frac{C}{|\mu_0 - h_{min}|}, \frac{R - C}{|h_{max} - \mu_0|} \right)$ and if  $\mu_0=h_{min}$, then $h_{Q}\geq \mu_0=h_{min}$ everywhere. So $\alpha h_Q+\beta\geq \alpha \mu_0+\beta=C\geq 0$, therefore the lower bound is trivially 0 in this case and we can pick any $\alpha\in(0,\frac{R-C}{h_{max}-\mu_0}]$. Hence is a valid $\Pi$ implementable mechanism.

\textbf{Second Part: Thresholding rule are implementable iff $r$ is quasi convex and lower semi-continuous}\\~\\
We now show the second part of the Proof of Theorem ~\ref{theorem:obedience-to-credal} which states that there exists an implementable regulation mechanism $\Pi$ for a requirement obtained via thresholding a metric $r$ if and only if the metric is quasi-convex and lower semi-conitnuous.

$(\Rightarrow)$

Assume that $\mathfrak{R}_\tau$ is implementable for all $\tau \in \mathbb{R}$. By first part of the Theorem~\ref{theorem:obedience-to-credal}, this implies that the set $\mathcal{P}^\tau_0 = \{ P \mid r(P) \le \tau \}$ is a Credal Set (convex and closed) for every $\tau$. 

Consider any two distributions $P_1, P_2 \in \Delta(\mathcal{Z})$ and any $\lambda \in [0, 1]$. Let $\tau^* = \max(r(P_1), r(P_2))$. By construction, $r(P_1) \le \tau^*$ and $r(P_2) \le \tau^*$, which implies $P_1 \in \mathcal{P}^{\tau^*}_0$ and $P_2 \in \mathcal{P}^{\tau^*}_0$. Since $\mathfrak{R}_{\tau^*}$ is implementable, $\mathcal{P}^{\tau^*}_0$ is convex. Therefore, the mixture $P_\lambda = \lambda P_1 + (1-\lambda)P_2$ must also belong to $\mathcal{P}_{\tau^*}$. 
\begin{align*}
P_\lambda \in \mathcal{P}_{\tau^*} \implies r(P_\lambda) \le \tau^*=\max(r(P_1), r(P_2)).
\end{align*}
Therefore $r$ is quasi-convex. By the initial assumption, $\mathcal{P}^\tau_0$ is a closed set for all $\tau$. Since the sublevel sets of $r$ are closed for all $\tau$, $r$ is lower semi-continuous by definition.

$(\Leftarrow)$

Assume that $r$ is quasi-convex and lower semi-continuous. We must show that $\mathcal{P}^\tau_0$ is a Credal Set for any arbitrary $\tau \in \mathbb{R}$. We now show the convexity of $\mathcal{P}^\tau_0$. Let $P_1, P_2 \in \mathcal{P}^\tau_0$. By definition, $r(P_1) \leq \tau$ and $r(P_2) \leq \tau$. Because $r$ is quasi-convex: 
\begin{align*}
r(\lambda P_1 + (1-\lambda)P_2) \leq \max(r(P_1), r(P_2)) \leq \tau.
\end{align*}
Thus, any convex combination of points in $\mathcal{P}^\tau_0$ remains in $\mathcal{P}^\tau_0$. The set is convex. Since $r$ is lower semi-continuous, its sublevel sets are closed by definition. Thus, $\mathcal{P}^\tau_0$ is closed. Since $\mathcal{P}^\tau_0$ is both convex and closed, it is a Credal Set. Therefore by Theorem~\ref{theorem:obedience-to-credal}, $\mathfrak{R}_\tau$ is implementable for all $\tau$.
\section{Characterisation and Properties of Obedient Regulations}
\subsection{Proof of Theorem~\ref{prop:characterisation-and-invariance-of-obedience}}
\begin{proposition}[Characterisation of Obedient Regulations]
     Given a set marginally undesirable gambles $\mathfrak{G}_{\leq 0,P}$ with respect to $\mathcal{P}_0$ and $\mathfrak{D}_{\geq 0}^{R}$ be the set of all desirable gambles with max payout of $R$, we can characterise the set of all obedient regulations with respect to $\mathcal{P}_0$ as 
     \begin{align*}
         \Pi^{\text{ob}}_{\mathcal{P}_0}= \{\mathfrak{G}_{\leq0,\mathcal{P}_0}+C\}\cap \mathfrak{D}_{\geq 0}^R
     \end{align*}
     where the set $\{\mathfrak{G}_{\leq0,\mathcal{P}_0}+C\}:= \{g+c|g\in\mathfrak{G}_{\leq 0,\mathcal{P}_0}\}$. Additionally, $\Pi^{\text{ob}}_{\mathcal{P}_0}$ is invariant up to the convex hull of $\mathcal{P}_0$ i.e. $\text{co}(\mathcal{P}_0)$. Formally, $\Pi_{\mathcal{P}_0}=\Pi_{\text{co}(\mathcal{P}_0})=\Pi_{\mathcal{Q}}$, where $\text{co}(\cdot)$ is the convex hull of a set and $\mathcal{P}_0\subseteq\mathcal{Q}\subseteq\text{co}(\mathcal{P}_{0})$. 
\end{proposition}
\begin{proof}

$(\Rightarrow)$

We want to show that $\Pi^{\text{ob}}_{\mathcal{P}_0}\subseteq \{\mathfrak{G}_{\leq0,\mathcal{P}_0}+C\}\cap \mathfrak{D}_{\geq 0}$. Let us assume that there exists a $\pi'\in\Pi^{\text{ob}}_{\mathcal{P}_0}$. 
Since $\pi'\in\Pi^{\text{ob}}_{\mathcal{P}_0}$ which is a set of all non-negative licenses licesnes bounded by $R$, $\pi':Z\rightarrow[0,1]$. Therefore, $\pi'\in\mathfrak{D}_{\geq 0}^R$ by definition. This follows from the fact that, $\mathfrak{D}_{\geq 0}^R$ is set of all desirable gambles which have a non-negative output and max payout of $R$, i.e. $\forall g\in\mathfrak{D}_{\geq 0}^R$, $g:Z\rightarrow[0,R]$. 

Since $\pi'\in\Pi^{\text{ob}}_{\mathcal{P}_0}$, it follows from the Definition of obedience that, 
\begin{align*}
    E_{P}[\pi'(Z)]&\leq C \quad \forall P\in\mathcal{P}_0\\
    \sup_{P\in\mathcal{P}_0}E_{P}[\pi'(Z)-C]&\leq 0\\
    \implies \quad \pi'-C &\in \mathfrak{G}_{\leq0,\mathcal{P}_0}\\
    \implies \quad \pi'&\in \{\mathfrak{G}_{\leq0,\mathcal{P}_0}+C\}
\end{align*}

$(\Leftarrow)$

We want to show that $\{\mathfrak{G}_{\leq0,\mathcal{P}_0}+C\}\cap \mathfrak{D}_{\geq 0}^R\subseteq\Pi^{\text{ob}}_{\mathcal{P}_0}$. Let us assume that there exists a $g\in\{\mathfrak{G}_{\leq0,\mathcal{P}_0}+C\}\cap \mathfrak{D}_{\geq 0}$. Since $g\in\mathfrak{D}_{\geq 0}^R$ we know that $g:Z\rightarrow[0,R]$. We also that $g\in\{\mathfrak{G}_{\leq0,\mathcal{P}_0}+C\}$ which means that $g-C\in \mathfrak{G}_{\leq0,\mathcal{P}_0}$. Therefore, 
\begin{align*}
    \sup_{P\in\mathcal{P}_0}\mathbb{E}_{P}[g(Z)-C]&\leq 0 \quad \tag*{(By Definition of \mbox{$g-C\in \mathfrak{G}_{\leq0,\mathcal{P}_0}$})}\\
    \sup_{P\in\mathcal{P}_0}\mathbb{E}_{P}[g(Z)]\leq C \quad \text{ and }\quad g:Z&\rightarrow[0,R] \quad \implies g\in\Pi^{\text{ob}}_{\mathcal{P}_0}
\end{align*}
Hence $\Pi^{\text{ob}}_{\mathcal{P}_0}= \{\mathfrak{G}_{\leq0,\mathcal{P}_0}+C\}\cap \mathfrak{D}_{\geq 0}^R$. We now prove the second part of the proof related to invariance of set of obedience regulations upto the convex hull of $\mathcal{P}_0$.

\textbf{Claim 1: (\mbox{$\Pi^{\text{ob}}_{\mathcal{P}_0}$} is a convex set)} To prove $\Pi^{\text{ob}}_{\mathcal{P}_0}$ is a convex set, we must show that a convex mixture of any two obedient licenses is also obedient to regulation.
Let $\pi_1, \pi_2 \in \Pi^{\text{ob}}_{\mathcal{P}_0}$. By definition, this means:$$\forall P \in \mathcal{P}_0, \quad \mathbb{E}_P[\pi_1] \leq C \quad \text{and} \quad \mathbb{E}_P[\pi_2] \leq C$$
Consider a mixture $\pi_{\lambda} = \lambda \pi_1 + (1-\lambda)\pi_2$ for any $\lambda \in [0,1]$.
We test if $\pi_{\lambda}$ satisfies the constraint for an arbitrary $P \in \mathcal{P}_0$:$$\mathbb{E}_P[\pi_{\lambda}] = \mathbb{E}_P[\lambda \pi_1 + (1-\lambda)\pi_2]$$

By the linearity of the expectation operator (w.r.t the function):$$\mathbb{E}_P[\pi_{\lambda}] = \lambda \mathbb{E}_P[\pi_1] + (1-\lambda) \mathbb{E}_P[\pi_2]$$

Since $\lambda \geq 0$ and $(1-\lambda) \geq 0$, we can apply the inequalities from Step 1:$$\lambda \mathbb{E}_P[\pi_1] + (1-\lambda) \mathbb{E}_P[\pi_2] \leq \lambda C + (1-\lambda) C$$$$= C(\lambda + 1 - \lambda) = C$$

Therefore, $\sup_{P \in \mathcal{P}_0} \mathbb{E}_P[\pi_{\lambda}] \leq C$. The mixture $\pi_{\lambda}$ is in $\Pi$.

\textbf{Claim 1: (\mbox{$\Pi$} is a invariant up to convex hull of \mbox{$\mathcal{P}_0$})} Additionally, we want to show that $\Pi$ is invariant upto convex hull of $\mathcal{P}_0$. We claim that the set of licenses obedient to $\mathcal{P}_0$ is identical to the set of licenses obedient to $\text{co}(\mathcal{P}_0)$.$$\Pi_{\mathcal{P}_0} = \Pi_{\text{co}(\mathcal{P}_0)}$$

($\Leftarrow$)

Since $\mathcal{P}_0 \subseteq \text{co}(\mathcal{P}_0)$, any constraint that applies to the larger set $\text{co}(\mathcal{P}_0)$ automatically applies to the subset $\mathcal{P}_0$.
If $\sup_{Q \in \text{co}(\mathcal{P}_0)} \mathbb{E}_Q[\pi] \leq C$, then trivially $\sup_{P \in \mathcal{P}_0} \mathbb{E}_P[\pi] \leq C$.
Thus, $\Pi_{\text{co}(\mathcal{P}_0)} \subseteq \Pi_{\mathcal{P}_0}$.

($\Rightarrow$)

Let $Q$ be any distribution in the convex hull $\text{co}(\mathcal{P}_0)$. By definition of a convex hull, $Q$ is a finite convex combination of elements in $\mathcal{P}_0$:$$Q = \sum_{i=1}^n \alpha_i P_i$$where $P_i \in \mathcal{P}_0$, $\alpha_i \geq 0$, and $\sum \alpha_i = 1$. Evaluate the expected cost of $\pi$ under $Q$:$$\mathbb{E}_Q[\pi] = E_{\sum \alpha_i P_i}[\pi]$$By the linearity of the expectation operator (w.r.t the measure):$$\mathbb{E}_Q[\pi] = \sum_{i=1}^n \alpha_i E_{P_i}[\pi]$$Since $\pi \in \Pi_{\mathcal{P}_0}$, we know that $E_{P_i}[\pi] \leq C$ for all $i$. Substituting this bound:$$\sum_{i=1}^n \alpha_i E_{P_i}[\pi] \leq \sum_{i=1}^n \alpha_i C$$$$= C \sum_{i=1}^n \alpha_i = C$$ Therefore, $\mathbb{E}_Q[\pi] \leq C$ for all $Q \in \text{co}(\mathcal{P}_0)$.
\end{proof}

\subsection{Proof of Alternate Characterisation in Proposition ~\ref{prop:characterisation-and-invariance-of-obedience}}
\begin{proof}
    We prove the alternate characterisation by contradiction. Let us assume that is $\Pi^{\text{ob}}_{\mathcal{P}_0}$ the collection of all licenses that are obedient to regulation and is therefore itself obedient to regulation by definition. However, $$\Pi^{\text{ob}}_{\mathcal{P}_0}\neq \left\{\pi\middle| \sup_{P\in\mathcal{P}_0}\mathbb{E}_{P}[\pi(Z)]\leq C\right\}:=\Pi_{C,\mathcal{P}_0}.$$ This means that there exists a $\pi\in\Pi^{\text{ob}}_{\mathcal{P}_0}$ such that $\pi\not\in\Pi_{C,\mathcal{P}_0}$. Therefore, 
    \begin{align*}
        \sup_{P\in\mathcal{P}_0} \mathbb{E}[\pi(Z)]>C
    \end{align*}
    However, since $\Pi^{\text{ob}}_{\mathcal{P}_0}$ follows obedience to regulation we can say that
    \begin{align*}
       \sup_{\pi'\in\Pi^{\text{ob}}_{\mathcal{P}_0}} \mathbb{E}[\pi'(Z)]&\leq C \quad \forall P\in\mathcal{P}_0\\
        \mathbb{E}[\pi(Z)]&\leq C \quad \forall P\in\mathcal{P}_0 \tag*{(Because \mbox{$\pi\in\Pi^{\text{ob}}_{\mathcal{P}_0}$})}\\
        \sup_{P\in\mathcal{P}_0} \mathbb{E}[\pi(Z)]&\leq C
    \end{align*}
    This leads to a contradiction hence $\Pi^{\text{ob}}_{\mathcal{P}_0}$ must equal $\Pi_{C,\mathcal{P}_0}$ if it is the set of all obedient regulations. 
\end{proof}

\subsection{Properties of the Mechanism of All Obedient Regulations 
$\Pi^{\text{ob}}_{\mathcal{P}_0}$}
\label{subsection:properties_of_obedient_regulations}

We now proof the additional claims we make about $\Pi^{\text{ob}}_{\mathcal{P}_0}$.
\begin{proposition}
    $\Pi^{\text{ob}}_{\mathcal{P}_0}$ is closed in weak-topology on $\mathcal{C}(\mathcal{Z})$. 
\end{proposition}
\begin{proof}
    We start with the alternate characterisation of $\Pi^{\text{ob}}_{\mathcal{P}_0}$ from Theorem~\ref{prop:characterisation-and-invariance-of-obedience}, i.e. \begin{align*}\Pi^{\text{ob}}_{\mathcal{P}_0}&=\left\{\pi\middle| \sup_{P\in\mathcal{P}_0}\mathbb{E}_{P}[\pi(Z)]\leq C\right\}\\
    &= \bigcap_{P\in\mathcal{P}_0}\left\{\pi\middle| \mathbb{E}_{P}[\pi(Z)]\leq C\right\}\\
    &= \bigcap_{P\in\mathcal{P}_0} A_P
    \end{align*}
The set of obedient regulations is the intersection of half planes $A_P$ for each $P\in\mathcal{P}_0$. Therefore, $\Pi^{\text{ob}}_{\mathcal{P}_0}$ is closed if $A_P$ is closed for every $P\in\mathcal{P}_0$ as arbitrary intersections of closed sets is closed. We now show that in our space of continuous functions $\mathcal{C}(\mathcal{Z})$ with corresponding weak topology induced by the dual space of probability measures $\Delta(\mathcal{Z})$, all half spaces such as $A_P$ are closed. Given an arbitrary $A_P$ half space, let us now consider a sequence of $\{\pi_n\}_{n=1}^\infty\in A_P$ such that $\{\pi_n\}_{n=1}^\infty\to \pi$. Then, weak-topology on $\mathcal{C}(\mathcal{Z})$ implies convergence of the expectations $E_{P}[\pi_n] \to E_P[\pi]$ for every $P\in\Delta(\mathcal{Z})$. As $\{\pi_n\}_{n=1}^\infty \in A_P$, we know that $E_{P}[\pi_n] \leq C$ for all $n$. Since the limit of a sequence of numbers less than or equal to $C$ must also be less than or equal to $C$: 
\begin{align*}
   E_P[\pi] = \lim_{n \to \infty} E_{P}[\pi_n] \leq C 
\end{align*}
Hence, $\pi \in A_P$ and $A_P$ is closed. Therefore $\Pi^{\text{ob}}_{\mathcal{P}_0}$ is closed. 
\end{proof}
\section{Proof of Lemma~\ref{lemma:all-obidient-or-nothing}}
\begin{lemma}
    Let $\mathcal{P}_0$ be a credal set. If $\Pi$ is implementable and $\pi$ is any license such that $\{\pi\}$ satisfies obedience, then the expanded mechanism $\Pi \cup \{\pi\}$ is also implementable. Consequently, the set of all obedient licenses $\Pi^{\textnormal{obd}}_{\mathcal{P}_0}$ is inherently implementable.    
    \label{lemma:all-obidient-or-nothing-1}
\end{lemma}
\begin{proof}
    We begin by showing the first part of the proof. Let $\mathcal{P}_0$ be a credal set. Then according to Theorem~\ref{theorem:obedience-to-credal} there exists an implementable regulation mechanism $\Pi\subseteq \mathcal{C}(\mathcal{Z})$. We also consider a $\pi$ such that $\{\pi\}$ satisfies obedience. If $\pi\in\Pi$, then $\Pi\cup\{\pi\}=\Pi$. Hence, $\Pi\cup\{\pi\}$ is implementable. Now, let's consider the case when $\pi\notin\Pi$. In this case $\Pi\cup\{\pi\}$ is still an obedient mechanism because every element in the set satisfies obedience by definition. $\Pi\cup\{\pi\}$ also satisfies feasibility because $\Pi$ satisfies feasibility.

    Given the first part of the proof is done, all we need to show is the following. If $\Pi^{\textnormal{obd}}_{\mathcal{P}_0}$ is not implementable then does not exists another mechanism that is implementable. Additionally, when $\mathcal{P}_0$ is a credal set, $\Pi^{\textnormal{obd}}_{\mathcal{P}_0}$ is the largest implementable mechanism.
    
    Let us assume that $\Pi^{\text{obd}}_{\mathcal{P}_0}$ is not implementable and there exists a set of continuous functions $\Pi=\{\pi:Z\rightarrow\mathbb{R}_{\geq 0}\}$ that is implementable. This means that, either $\Pi\subset \Pi^{\text{obd}}_{\mathcal{P}_0}$ or $\Pi\not\subseteq \Pi^{\text{obd}}_{\mathcal{P}_0}$. Let's consider these two cases separately. 
    \begin{itemize}
        \item \textbf{Case I: \mbox{$\Pi\not\subseteq \Pi^{\textnormal{obd}}_{\mathcal{P}_0}$}} 

    Since $\Pi\not\subseteq \Pi^{\text{obd}}_{\mathcal{P}_0}$, this implies that there exists a $\pi'\in\Pi$ such that $\pi'\not\in\Pi^{\text{obd}}_{\mathcal{P}_0}$. As $\Pi$ is implementable it must satisfy obedience to regulation, i.e. 
    \begin{align}
        \sup_{\pi\in\Pi}\mathbb{E}_P[\pi(Z)]&\leq C \quad \forall P\in\mathcal{P}_0\tag*{(By Definition~\ref{def:obidience_to_regulation})}\\
        \implies \quad \mathbb{E}_P[\pi'(Z)]&\leq C \quad \forall P \in \mathcal{P}_0\\
        \implies \pi'\in\Pi^{\textnormal{obd}}_{\mathcal{P}_0}\quad \tag*{(By Definition of \mbox{$\Pi^{\textnormal{obd}}_{\mathcal{P}_0}$})}
    \end{align}
    This leads to a contradiction, therefore any $\Pi\not\subseteq \Pi^{\text{obd}}_{\mathcal{P}_0}$ cannot be implementable if $\Pi^{\text{obd}}_{\mathcal{P}_0}$ is not implementable.
    \item \textbf{Case II: \mbox{$\Pi\subset \Pi^{\textnormal{obd}}_{\mathcal{P}_0}$}}

    Since $\Pi\subseteq \Pi^{\textnormal{obd}}_{\mathcal{P}_0}$. $\Pi$ is obedient to regulations by definition. Since $\Pi$ is implementable it must also satisfy feasibility. Which means that for every $P\in\Delta(\mathcal{Z})\setminus\mathcal{P}_0$ there exists a $\pi\in\Pi$ such that $\mathbb{E}_{P}[\pi(Z)]>C$. However, since $\Pi\subset\Pi^{\textnormal{obd}}_{\mathcal{P}_0}$. Any such $\pi$ must also belong to $\Pi^{\textnormal{obd}}_{\mathcal{P}_0}$. Which makes $\Pi^{\textnormal{obd}}_{\mathcal{P}_0}$ feasible. Since $\Pi^{\textnormal{obd}}_{\mathcal{P}_0}$ is obedient by definition, it is also implementable. This leads to a contradiction, and therefore $\Pi\subset \Pi^{\textnormal{obd}}_{\mathcal{P}_0}$ cannot be implementable if $\Pi^{\textnormal{obd}}_{\mathcal{P}_0}$ is not implementable.
    \end{itemize}
    Hence, if $\Pi^{\text{obd}}_{\mathcal{P}_0}$ is not implementable there does not exist any implementable mechanism. We now move on the last part of our proof where we show that when $\mathcal{P}_0$ is the credal set, i.e. when an implementable mechanism exists, $\Pi^{\text{obd}}_{\mathcal{P}_0}$ is the largest implementable mechanism. Let us assume that $\mathcal{P}_0$ is a credal set which means that there exists an implementable mechanism $\Pi$ and this mechanism is bigger than $\Pi^{\text{obd}}_{\mathcal{P}_0}$. Then implies that  $\Pi\neq\Pi^{\text{obd}}_{\mathcal{P}_0}$ and $\Pi\not\subseteq\Pi^{\text{obd}}_{\mathcal{P}_0}$, and hence $\exists\pi\in\Pi$ such that $\pi\not\in\Pi^{\text{obd}}_{\mathcal{P}_0}$. Therefore,
    \begin{align*}
        \exists P\in\mathcal{P}_0 \quad \text{ s.t. }\mathbb{E}_{P}[\pi(Z)]> C
    \end{align*}
    This directly contradicts that $\Pi$ satisfies Obedience to Regulation (Defintion~\ref{def:obidience_to_regulation}). Hence $\Pi^{\text{obd}}_{\mathcal{P}_0}$ is the largest implementable mechanism when $\mathcal{P}_0$ is a credal set.  
\end{proof}
\section{Proof of Proposition \ref{prop:credal-implementable-menu}}
\begin{proof}
    Take any arbitrary license $\pi \in \Pi$. By the definition of the mechanism, this license is constructed from some test $\phi \in \bm{\Phi}$ such that $\pi = R \cdot \phi$ where every test $\phi \in \bm{\Phi}$ is a credal test with size $\alpha \le \frac{C}{R}$. By the definition of the size of a test over a composite null hypothesis $\mathcal{P}_0$:
    $$\sup_{P \in \mathcal{P}_0} \mathbb{E}_P[\phi] \le \frac{C}{R}$$
    
    Because the expectation operator is linear and the reward $R$ is a positive constant, we can multiply both sides of the inequality by $R$:
    $$R \cdot \sup_{P \in \mathcal{P}_0} \mathbb{E}_P[\phi] \le R \cdot \left( \frac{C}{R} \right)$$$$\sup_{P \in \mathcal{P}_0} \mathbb{E}_P[R \cdot \phi] \le C$$Substituting $\pi = R \cdot \phi$, we obtain:$$\sup_{P \in \mathcal{P}_0} \mathbb{E}_P[\pi] \le C$$Since this holds true for every $\pi \in \Pi$, the mechanism strictly guarantees Obedience. Now, take any arbitrary compliant model $Q \notin \mathcal{P}_0$, then there exists at least one test $\phi \in \bm{\Phi}$ that is unbiased against $Q$. By the statistical definition of an unbiased test, its statistical power against the alternative hypothesis $Q$ strictly exceeds its size $\alpha$:
    $$\mathbb{E}_Q[\phi] > \alpha$$
    Since the proposition explicitly restricts the set to tests where $\alpha = \frac{C}{R}$, we substitute this value into the inequality:$$\mathbb{E}_Q[\phi] > \frac{C}{R}$$Multiplying both sides by the positive constant $R$ yields:$$R \cdot \mathbb{E}_Q[\phi] > R \cdot \left( \frac{C}{R} \right)$$$$\mathbb{E}_Q[R \cdot \phi] > C$$Letting $\pi = R \cdot \phi$, we have mathematically identified a specific license $\pi \in \Pi$ such that:$$\mathbb{E}_Q[\pi] > C$$Since such a license exists for every strictly compliant model $Q \notin \mathcal{P}_0$, the mechanism guarantees Feasibility. Hence $\Pi$ is implementable.
\end{proof}
\section{Proof of Proposition~\ref{propositon:sequential-implementable-menu}}
Before going onto the main proof we will show a simple identity, 
\begin{proposition}
    Let $a$ be a constant such that $\text{ln}(a)>1$ and $P$ and $Q$ be two distinct distributions in $\Delta(\mathcal{Z})$, then
    \begin{align*}
        \mathbb{E}_{Q}\Bigg[\log\bigg(\min\bigg\{\frac{Q(z)}{P(z)},a\bigg\}\bigg)\Bigg]>0
    \end{align*}
\label{prop:extra-for-prop4.3}
\end{proposition}
\begin{proof}
    We now take the left hand since of the identity 
    \begin{align*}
        \mathbb{E}_{Q}\Bigg[\log\bigg(\min\bigg\{\frac{Q(z)}{P(z)},a\bigg\}\bigg)\Bigg]&=\int\log\bigg(\min\bigg\{\frac{Q(z)}{P(z)},a\bigg\}\bigg)\frac{Q(z)}{P(z)}P(z)dz\\
        &=\mathbb{E}_{P}\Bigg[\frac{Q(z)}{P(z)}\log\bigg(\min\bigg\{\frac{Q(z)}{P(z)},a\bigg\}\bigg)\Bigg]\\
        &=\mathbb{E}_{P}\Bigg[\frac{Q(z)}{P(z)}\log\bigg(\min\bigg\{\frac{Q(z)}{P(z)},a\bigg\}\bigg)-\frac{Q(z)}{P(z)}+1\Bigg]\tag*{\Bigg(\mbox{$\mathbb{E}_{P}[\frac{Q(z)}{P(z)}]=1$}\Bigg)}
    \end{align*}
We now analyse the expression $\frac{Q(z)}{P(z)}\log\bigg(\min\bigg\{\frac{Q(z)}{P(z)},a\bigg\}\bigg)-\frac{Q(z)}{P(z)}+1$ by looking at $\frac{Q(z)}{P(z)}=x$ where we consider $x\in(0,\infty)$ in the following two cases. When $x<a$, then the expression $x\log(x)-x+1$ is convex and has minimum at $x=1$ as $0$ which is not attained uniformly when $P$ and $Q$ are distinct. When $x\geq a$, then the expression is $x(\log(a)-1)+1$. Which is positive if $\log(a)>1$, or $a\gtrsim 2.73$. Hence the expression inside the expectation is always positive, which proves the identity.
\end{proof}
\begin{proof}
We now show that $\Pi=\{C\cdot\phi_{\tilde{P}} \mid \tilde{P}\in\Delta(\mathcal{Z})\setminus\mathcal{P}_0\}$
satisfies obedience to regulation, where $\phi_{\tilde{P}}(z)=\min\left\{\frac{\tilde{P}(z)}{P^*(z)},\,\frac{R}{C}\right\}$
and $P^*=\arg\min_{P\in\mathcal{P}_0}\mathrm{KL}(Q\|P)$ for an agent with type $Q$.

Let us assume an arbitrary non-compliant agent $P'\in\mathcal{P}_0$.
We now show that $\Pi$ is obedient to regulation:
\begin{align*}
\sup_{\pi\in\Pi}\mathbb{E}_{P'}[\pi(z)]
&=
\sup_{\tilde{P}\in\Delta(\mathcal{Z})\setminus\mathcal{P}_0}
C\,\mathbb{E}_{P'}
\left[\min\left\{\frac{\tilde{P}(z)}{P^*(z)},\,\frac{R}{C}\right\}\right] \\
&\leq\sup_{\tilde{P}\in\Delta(\mathcal{Z})\setminus\mathcal{P}_0}C \min\left\{\mathbb{E}_{P'}\left[\frac{\tilde{P}(z)}{P^*(z)}\right],\,\frac{R}{C}\right\}.
\end{align*}
Since $P'\in\mathcal{P}_0$ when minimizing $P^*=\arg\min_{P\in\mathcal{P}_0}\mathrm{KL}(P'\|P)$, we obtain $P^*=P'$. Hence, 
\begin{align*}
    \sup_{\tilde{P}\in\Delta(\mathcal{Z})\setminus\mathcal{P}_0}C \min\left\{\mathbb{E}_{P'}\left[\frac{\tilde{P}(z)}{P^*(z)}\right],\,\frac{R}{C}\right\}&=\sup_{\tilde{P}\in\Delta(\mathcal{Z})\setminus\mathcal{P}_0}C \min\left\{\mathbb{E}_{P'}\left[\frac{\tilde{P}(z)}{P'(z)}\right],\,\frac{R}{C}\right\}\\
    &= \sup_{\tilde{P}\in\Delta(\mathcal{Z})\setminus\mathcal{P}_0}C \min\left\{1,\,\frac{R}{C}\right\}\\
    &= C \tag*{(\mbox{$R<C$})}
\end{align*}
Hence $\sup_{\pi\in\Pi}\mathbb{E}_{P'}[\pi(z)]<C$ for any arbitrary non-compliant agent with type $P'\in\mathcal{P}_0$. Therefore, $\Pi$ satisfies obedience to regulation. We now show that $\Pi$ also satisfies feasibility of regulation. We know that $R/C\gtrsim 2.73$ i.e. $\text{ln}(R/C)>1$, then according to Proposition~\ref{prop:extra-for-prop4.3} we can say that for all $Q\in\Delta(\mathcal{Z})\setminus\mathcal{P}_0$ and $P\in\Delta(\mathcal{Z})$ such that $Q\neq P$ 
\begin{align*}
        \mathbb{E}_{Q}\Bigg[\log\bigg(\min\bigg\{\frac{Q(z)}{P(z)},\frac{R}{C}\bigg\}\bigg)\Bigg]&>0\\
        \implies \log\Bigg(\mathbb{E}_{Q}\Bigg[\min\bigg\{\frac{Q(z)}{P(z)},\frac{R}{C}\bigg\}\Bigg]\Bigg)&>0 \tag*{(Jensen's Inequality)}\\
        \mathbb{E}_{Q}\Bigg[\min\bigg\{\frac{Q(z)}{P(z)},\frac{R}{C}\bigg\}\Bigg]&>1\\
        \mathbb{E}_{Q}\Bigg[C\min\bigg\{\frac{Q(z)}{P(z)},\frac{R}{C}\bigg\}\Bigg]&>C
    \end{align*}
Since $Q\in\Delta(\mathcal{Z})\setminus\mathcal{P}$, therefore $P^*$ is always guaranteed to be distinct from $Q$ and hence, 
\begin{align*}
    \mathbb{E}_{Q}\Bigg[C\min\bigg\{\frac{Q(z)}{P^*(z)},\frac{R}{C}\bigg\}\Bigg]&>C
\end{align*}
where $P^*=\argmin_{P\in\mathcal{P}_0}\text{KL}(Q\|P)$. Which is true for all $Q\in\Delta(\mathcal{Z})\setminus\mathcal{P}$. Hence, $\Pi$ also satisfies feasibility of regulation. Therefore, $\Pi$ is implementable. 
\end{proof}
\section{Background on Operationalising Regulations with Implicit Credal Sets}
\label{app:martingale_background}
In this section, we do a quick review of the background concepts underpinning our proof of Proposition~\ref{proposition:matringale-implementable} which allows the regulators build implementable mechanisms with implicit credal sets. We focus specifically focusing on filtrations, martingales, and predictable processes. For more detailed technical exposition readers may refer to \citet{williams1991probability}.
\begin{definition}[Filtrations and Measurability] Given a probability space $(\Omega, \mathcal{F}, P)$, a filtration is a non-decreasing sequence of sub-$\sigma$-algebras $(\mathcal{F}_n)_{n \geq 0}$ such that:
\begin{equation}
    \mathcal{F}_0 \subseteq \mathcal{F}_1 \subseteq \dots \subseteq \mathcal{F}_n \subseteq \dots \subseteq \mathcal{F}.
\end{equation}
\end{definition}
Intuitively, $\mathcal{F}_n$ represents the information available at time $n$. In our context, if $Z_1, Z_2, \dots$ is the sequence of observed evidence, the natural filtration is defined as $\mathcal{F}_n = \sigma(Z_1, \dots, Z_n)$, with $\mathcal{F}_0$ being the trivial $\sigma$-algebra $\{\emptyset, \Omega\}$. A stochastic process $M = (M_n)_{n \geq 1}$ is said to be \emph{adapted} to the filtration $(\mathcal{F}_n)_{n \geq 0}$ if $M_n$ is $\mathcal{F}_n$-measurable for all $n$. A key concept in defining valid betting strategies is predictability. A process $M = (M_n)_{n \geq 1}$ is called predictable with respect to the filtration $(\mathcal{F}_n)_{n \geq 0}$ if $M_n$ is measurable with respect to the \textit{previous} time step's information, $\mathcal{F}_{n-1}$. Formally, for all $n \geq 1$, $M_n$ is $\mathcal{F}_{n-1}$-measurable. This property ensures that a betting strategy (such as the choice of $\lambda_n$) depends only on past observations and not on the outcome of the current round.
\begin{definition}[Martingales and Supermartingales]
Given a probability space $(\Omega, \mathcal{F}, P)$, a stochastic process $M = (M_n)_{n \geq 0}$ adapted to a filtration $(\mathcal{F}_n)_{n \geq 0}$ is called a martingale with respect to $P$ if, for all $n \geq 1$:
\begin{equation}
    \mathbb{E}_P[M_n \mid \mathcal{F}_{n-1}] = M_{n-1} \quad P\text{-almost surely.}
\end{equation}
Similarly, $M = (M_n)_{n \geq 0}$ is called a supermartingale with respect to $P$ if the expected future value is non-increasing:
\begin{equation}
    \mathbb{E}_P[M_n \mid \mathcal{F}_{n-1}] \le M_{n-1} \quad P\text{-almost surely.}
\end{equation}
\end{definition}
The martingale represents a ``fair game'' between a forecaster and the nature where the expected future value, given current information, equals the current value. Whereas super-martingales represent games that are unfavourable (or at best neutral) to the forecaster. By the tower property of conditional expectation, this implies $\mathbb{E}[M_n] \le \mathbb{E}[M_0]$ for all $n$. We encourage the reader to refer to \citet{vovk2005defensive} for a detailed exposition of the seqeuntial forecasting game interpretation. This allows martingales to be used in hypothesis testing, as a martingale $M_n$ with respect to a distribution $P$ at any step $n$ would in expectation not be more than $M_0$. Therefore, we can treat the value of $M_n$ to be the evidence against the null. However, in many settings the null hypothesis is not a single distribution but a set of distributions $\mathcal{P}_0$, often referred to as a composite null. To this end, we define a test (super)-martingale
\begin{definition}[Test Supermartingale \citep{shafer2011test}]
Let $(\Omega, (\mathcal{F}_n)_{n \geq 0})$ be a filtered measurable space, and let $\mathcal{P}_0$ be a set of probability measures on $(\Omega, \mathcal{F})$. Then, a process $(M_n)_{n \geq 0}$ is a test supermartingale with respect to the class $\mathcal{P}_0$ if:
\begin{enumerate}
    \item $M_n$ is non-negative and adapted to $\mathcal{F}_n$ for all $n \geq 0$,
    \item $M_0 = 1$ almost surely, and
    \item For every distribution $P \in \mathcal{P}_0$ and all $n \geq 1$, $M_n$ is a supermartingale under $P$:
    \begin{equation}
        \mathbb{E}_P[M_n \mid \mathcal{F}_{n-1}] \le M_{n-1} \quad P\text{-almost surely}.
    \end{equation}
\end{enumerate}
\end{definition}
This condition implies that $(M_n)$ yields valid evidence against the \textit{entire} set $\mathcal{P}_0$. Although not part of the original definition, it is easy to verify that $(M_n)_{n \geq 0}$ is a test supermatringale with respect to $\text{conv}(\mathcal{P}_0)$, where $\text{conv}(\cdot)$ denotes the convex hull. Thus with the closure of the convex hull, the process $(M_n)_{n \geq 0}$ tests against an implicit credal set $\overline{\text{conv}(\mathcal{P}_0)}$.

\subsection{Proof of Proposition~\ref{proposition:matringale-implementable}}
\begin{proof}
The the proof consists of two parts. First, we show that the mechanism satisfies \emph{obedience} (Definition~\ref{def:obidience_to_regulation}) by proving that $(\pi_n)_{n \geq 0}$ is a test supermartingale under the null hypothesis. Second, we show \emph{feasibility} (Definition~\ref{def:feasibility-of-regulation}) by demonstrating exponential growth under compliant distributions. Since the thresholding on $\pi_n$ occurs only once at the end when the license is provided, we can effectively ignore the threshold in our analysis for obedience and focus on the case where license value is less than $R$.

\textbf{Part 1: Obedience.} 
Let $\mathcal{F}_n = \sigma(z_1, \dots, z_n)$ be the natural filtration. Recall that our license evolves as a wealth process i.e. $\pi_n = \pi_{n-1} (1 + \lambda_n (h(z_n) - \tau))$. Consider any non-compliant distribution $P \in \mathcal{P}_0$. By Assumption~\ref{assumption:linearisable_requirements}, for any $P \in \mathcal{P}_0$, we have $\mathbb{E}_P[h(z)] \leq \tau$.

Taking the conditional expectation of $\pi_n$ with respect to $\mathcal{F}_{n-1}$:
\begin{align*}
    \mathbb{E}_P[\pi_n \mid \mathcal{F}_{n-1}] &= \mathbb{E}_P\left[ \pi_{n-1} \left(1 + \lambda_n (h(z_n) - \tau)\right) \mid \mathcal{F}_{n-1} \right] \\
    &= \pi_{n-1} \left( 1 + \lambda_n \left( \mathbb{E}_P[h(z_n) \mid \mathcal{F}_{n-1}] - \tau \right) \right).
\end{align*}
Here, $\pi_{n-1}$ and $\lambda_n$ are factored out of the expectation because $\pi_{n-1}$ is $\mathcal{F}_{n-1}$-measurable and the strategy $\lambda_n$ is predictable (determined by $\mathcal{F}_{n-1}$). Since $z_n$ is i.i.d., $\mathbb{E}_P[h(z_n) \mid \mathcal{F}_{n-1}] = \mathbb{E}_P[h(z)]$. Substituting the null constraint $\mathbb{E}_P[h(z)] - \tau \leq 0$:
\begin{equation*}
    \mathbb{E}_P[\pi_n \mid \mathcal{F}_{n-1}] \leq \pi_{n-1} (1 + \lambda_n \cdot 0) = \pi_{n-1}.
\end{equation*}
This inequality confirms that $(\pi_n)_{n \geq 0}$ is a non-negative super-martingale with respect to any $P \in \mathcal{P}_0$. By the defining property of super-martingales, $\mathbb{E}_P[\pi_n] \leq \mathbb{E}_P[\pi_0] = C$ for all $n \in \mathbb{N}$. Consequently, the mechanism satisfies obedience.

\textbf{Part 2: Feasibility.} Now consider a compliant distribution $Q$ such that $\mathfrak{R}(Q) = 1$. By Assumption~\ref{assumption:linearisable_requirements}, this implies that $\Delta_Q := \mathbb{E}_Q[h(z)] - \tau > 0$. We will show that there exists $N \in \mathbb{N}$ such that $\mathbb{E}_Q[\pi_n] > C$ for all $n \geq N$. 

Let $W_n := C \prod_{i=1}^n \big(1 + \lambda^*(h(z_i) - \tau)\big)$ denote the un-truncated wealth process under a fixed strategy $\lambda_n \equiv \lambda^*$, so that the issued license is $\pi_n = \min(W_n, R)$. We choose $\lambda^* \in (0, B]$ small enough that two conditions hold:
(1) \emph{Admissibility:} $1 + \lambda^*(h(z) - \tau) \geq 0$ almost surely, ensuring $W_n \geq 0$ and (2) \emph{Positive log-growth:} $\mathbb{E}_Q\!\left[\log\!\big(1 + \lambda^*(h(z) - \tau)\big)\right] > 0$. Both conditions can be simultaneously satisfied for sufficiently small $\lambda^* > 0$. Indeed, define $g(\lambda) := \mathbb{E}_Q[\log(1 + \lambda(h(z) - \tau))]$. Then $g(0) = 0$ and $g'(0) = \mathbb{E}_Q[h(z) - \tau] = \Delta_Q > 0$,
so $g(\lambda^*) > 0$ for all sufficiently small $\lambda^* > 0$. This is the standard Kelly-betting argument~\citep{kelly1956new}. Taking logarithms on $W_n$ we obtain,
\begin{align*}
\log W_n = \log C + \sum_{i=1}^n \log\!\big(1 + \lambda^*(h(z_i) - \tau)\big).
\end{align*}
Since the $z_i$ are i.i.d.\ under $Q$, the strong law of large numbers yields
\begin{align*}
\frac{1}{n} \log W_n \;\xrightarrow{n \to \infty}\; \mathbb{E}_Q\!\left[\log\!\big(1 + \lambda^*(h(z) - \tau)\big)\right] > 0 \quad Q\text{-almost surely}
\end{align*}
Therefore $\log W_n \to \infty$ and consequently $W_n \to \infty$ $Q$-almost surely. Since $\pi_n = \min(W_n, R)$ and $W_n \to \infty$ $Q$-almost surely, we have
\begin{align*}
\pi_n \;\xrightarrow{n \to \infty}\; R \quad Q\text{-almost surely.}
\end{align*}

Given that the licenses are uniformly bounded, i.e. $||\pi_n||_{\infty} \leq R$ for all $n$
\begin{align*}
\lim_{n \to \infty} \mathbb{E}_Q[\pi_n] \;=\; \mathbb{E}_Q\!\left[\lim_{n \to \infty} \pi_n\right] \;=\; R.
\end{align*}
Since $R > C$ by assumption, there exists $N \in \mathbb{N}$ such that $\mathbb{E}_Q[\pi_n] > C$ for all  $n \geq N$. Hence the mechanism $\Pi_n = \{\pi_n[\lambda]\}_\lambda$ satisfies feasibility for all $n \geq N$.
\end{proof}

\section{Connection to Strategic Hypothesis Testing} 
\label{appendix:hypothesis-testing}

We now formulate our regulation problem as a classical testing procedure in a non-parametrized fashion to highlight its differences from incentive aware statistical protocol introduced above. Our regulation can be formulated statistical decision making problem via a composite hypothesis test as follows:
\begin{align}
    H_0:\quad P\in \{Q\in\Delta(\mathcal{Z})\text{ s.t. }\mathfrak{R}(Q)=0\} \quad 
    H_1:\quad P\in \{Q\in\Delta(\mathcal{Z})\text{ s.t. }\mathfrak{R}(Q)=1\}
\label{eq:requirement-as-test}
\end{align}
Notice that the set of evidence distributions for both hypothesis $H_0$ and $H_1$ characterise the parameters of the evidence distributions arising according to the Definition~\ref{def:e-safe_model}. While classical hypothesis tests are a valid protocol to ensure regulation they do not incorporate the incentives of the model providers into the problem setup. To understand this better, let us assume that the regulator enforces the requirement $\mathfrak{R}$ with the hypothesis test defined in Equation~\ref{eq:requirement-as-test}. This hypothesis test will have some false positive rate, which we assume to be public information. With the knowledge of this false positive rate $\alpha$, strategic model providers are then incentivised to train and submit many borderline or bad models for which the true $P\in H_0$, when it justifies their cost-benefit calculus. The $\alpha$ false positive rate ensures that $\alpha$ percent of these models will get through the regulation. Thus the classic hypothesis tests often have incentives for the non-compliant model providers to strategise \citep{bates2022principal,hossain2025strategic}. We demonstrate this with a toy example for regulation below  
We consider a toy linear model for our simulations. We further assume that it aims to approximate an original data generating process of $y_i=x_i^T\theta^* + \epsilon$ where $\epsilon\sim\mathcal{N}(0,\sigma)$. Additionally, assume that all features affect the model prediction equally. We now wish to regulate the number of parameters / features used by this model.  Either the model designer could be using all the $d=d_0+1$ features to make the prediction, i.e., designer also uses the sensitive attribute to maximise their prediction capability or the designer could follow the regulation and only use non-sensitive attributes to make prediction. We frame the use of sensitive attribute as the null hypothesis and use of only non sensitive attribute as alternative to build an hypothesis test for regulation. We formalise the test as follows
\begin{align*}
H_0: &\;\; \text{Model is using sensitive attribute }  d=d_0+1,\\
H_1: &\;\; \text{Model is not using sensitive attribute } d=d_0.
\end{align*}
We consider the standardized quadratic error for features $X$ under OLS estimator $\hat{\theta}$ as the test statistic. That is $
Q_{\mathrm{std}}=\frac{n}{\sigma^2}Q=\frac{n}{\sigma^2} \, (\hat{\theta} - \theta^\ast)^\top (X X^\top)(\hat{\theta} - \theta^\ast)$ which is a $\chi^2_d$ distributed and
which under suitable regularity conditions follows a chi-square distribution with degrees of freedom equal to the effective number of parameters used in the model.
Since we know the parametric model of both null and alternative distributions and they are simple singleton hypothesis we can use a likelihood ratio test as it is the optimal test given Neyman Pearson Lemma. Let us denote the test statistic $L:=\frac{L(d_0+1;Q)}{L(d_0;Q)}$ where $Q\sim \chi^2_d$ and $L(d;Q)$ is the likelihood of sample $Q$ from chi-squared distribution with parameter $d$. 
Under $H_0$, assuming that the test statistic has a distribution
$P_{H_0}(L)$. We implement the test to reject $H_0$ if $L>\tau_\alpha$ where $\tau_\alpha$ is the $1-\alpha$ quantile of $P_{H_0}(L)$, so that we obtain strict $\alpha$ type 1 error. 
\paragraph{Strategic Aspects in the test}
The above testing procedure for enforcing regulation ignores the incentives to the model designers. However, in real world the model designers operate under incentives. In this section we consider some incentives that designers may have and try to understand their behaviour under a statistical test for regulation of use of sensitive attributes for training. Let us assume that for regulation tests the regulator charges a fee, this can also be understood as the tax to operate in the market, we denote it using $C$ and we assume that the size of the market is denoted by $R$. Ideally a regulation implemented by a test must deter null agents from entering the market, i.e. non obedient agents self opt out of the market while keeping the obedient agents in the market. 
\begin{figure}
    \centering
    \begin{subfigure}{0.26\textwidth}
        \includegraphics[width=\linewidth]{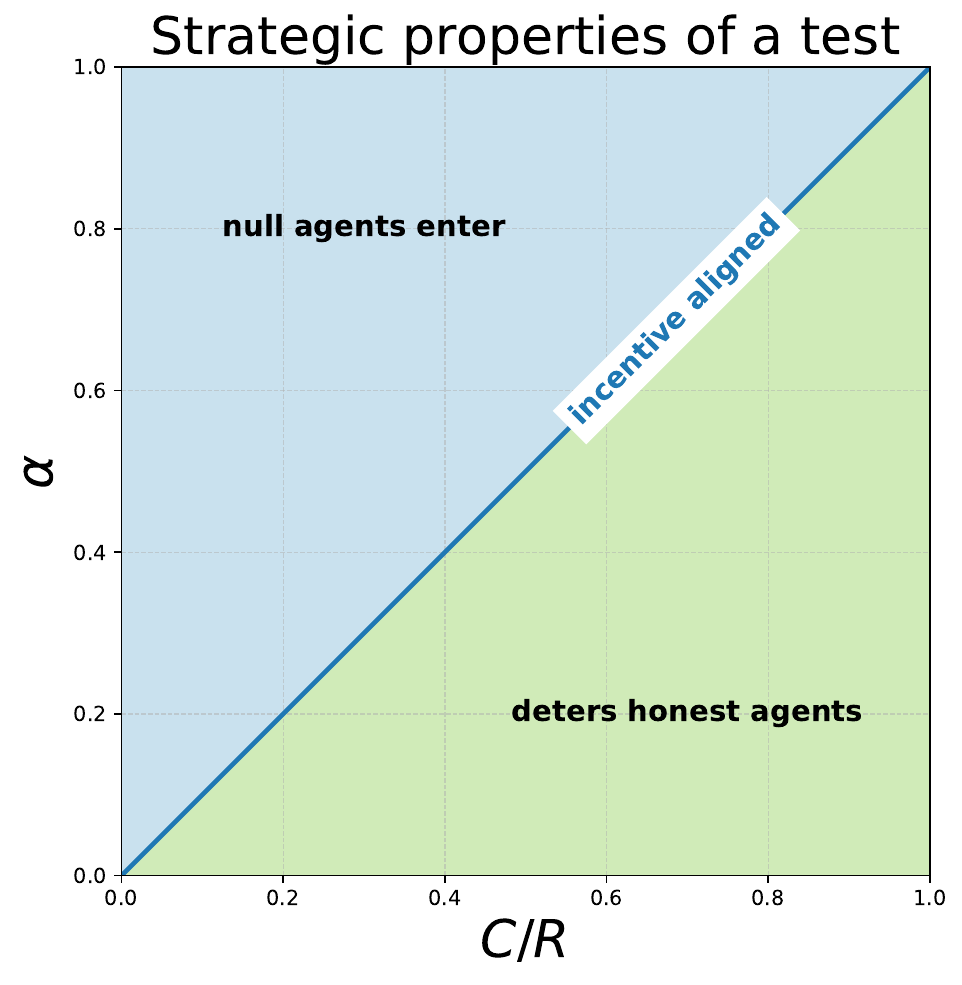}
        \caption{}
    \end{subfigure}
    \begin{subfigure}{0.35\textwidth}
        \includegraphics[width=\linewidth]{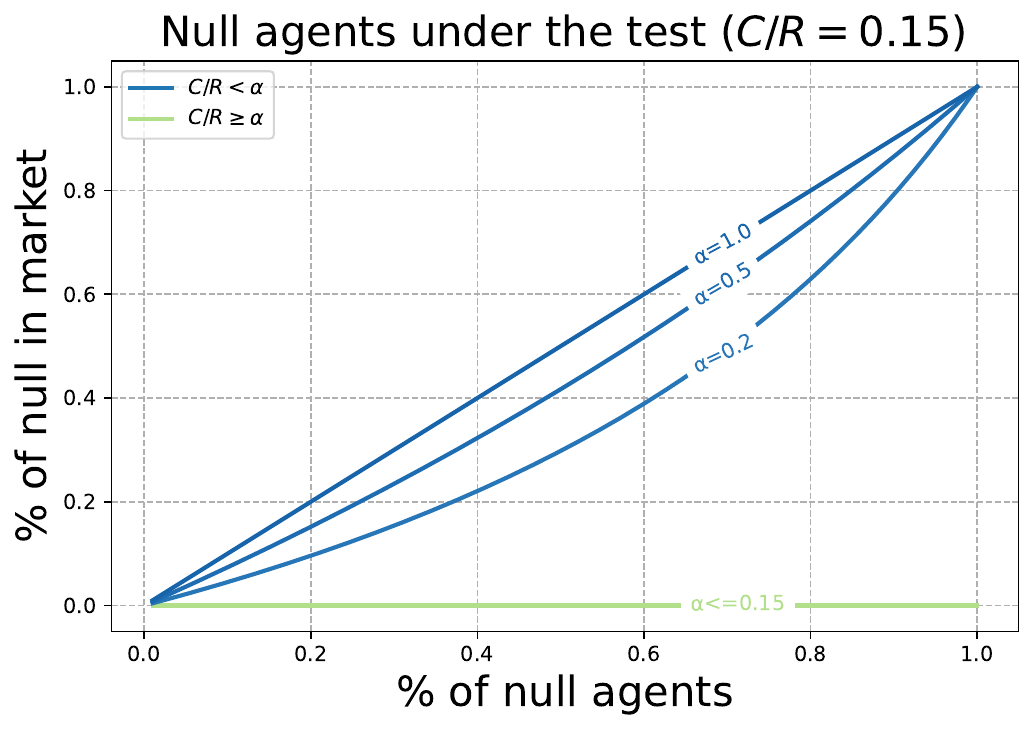}
        \caption{}
    \end{subfigure}
    \begin{subfigure}{0.35\textwidth}
        \includegraphics[width=\linewidth]{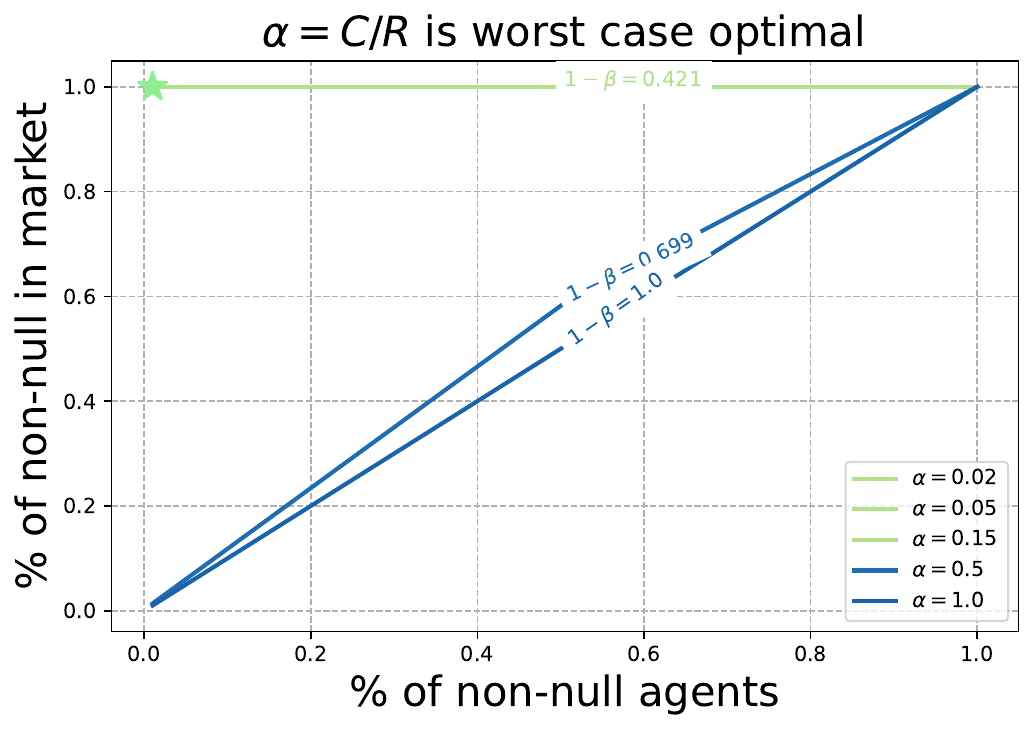}
        \caption{}
    \end{subfigure}
    
    \caption{The strategic reaction of null and non-null agents in the market to regulations via testing. The above figures (b) and (c) assume the incentives in the market by fixing $C/R=0.15$}
    \label{fig:threeplots}
\end{figure}
With the statistical test proposed above to check the effective dimension of the model and thus for the use of sensitive attribute, let us assume that the final test implemented by the regulator has false positive rate $\alpha$, the type II error $\beta(\alpha)$ is denoted as a function of the choice $\alpha$ made by the regulator thus the choice of false positive rate also dictates the power of the test $1-\beta(\alpha)$. In an Ex-ante analysis we can observe that null agents participation in the market depends directly on this $\alpha$, as for a null agent, $\alpha R\geq C$ means that the gamble to enter into the market has net positive expected utility. Thus for $\alpha>\frac{C}{R}$, the market will see full participation for approval by the null agents (see Fig~\ref{fig:threeplots} a), and because of test properties $\alpha$ proportion of the nul agents will also get approved (see Fig~\ref{fig:threeplots} b). Whereas for the non-null agents, the decision to participate in the market depends upon the power of the test i.e. $(1-\beta(\alpha))R\geq C$ which can be seen in the Fig~\ref{fig:threeplots} c that the too strict value of $\alpha$ results in a power below $0.15$ and thus lower than $C/R$ resulting in no participants in the market. As $\alpha$ gradually increases to $\alpha=0.15$ and thus equal to $C/R$ the power of the test increases resulting in more and more non-null agents being approved. 
\section{Additional Experimental Details}
\label{appendix:experiment-details}

\subsection{Details on the Waterbirds Experiment and Dataset}
\label{appendix:experiment-details-waterbirds}
We illustrate the proposed regulation mechanism on the Waterbirds dataset, a standard benchmark for studying spurious correlations in image classification. The dataset is constructed by superimposing bird images from CUB~\citep{wah2011caltech} onto background scenes from Places~\citep{zhou2017places}. The task is binary classification between landbirds and waterbirds. The training distribution exhibits strong spurious correlations: \(73\%\) of examples are landbirds with on land background and \(22\%\) are waterbirds on water, while counter-spurious groups (landbirds on water and waterbirds on land) comprise only \(4\%\) and \(1\%\) of the data, respectively. Validation and test splits are balanced across backgrounds to evaluate robustness. The regulator operates a licensing market in which agents must provide statistical evidence that their predictions do not rely on spurious background features.
\begin{table}[ht]
\centering
\caption{Waterbirds dataset distribution across different sub-groups.}
\resizebox{\textwidth}{!}{%
\begin{tabular}{lcccc}
\toprule
\textbf{Group Description} & \textbf{Landbird on Land} & \textbf{Landbird on Water} & \textbf{Waterbird on Land} & \textbf{Waterbird on Water} \\
\midrule
& \includegraphics[width=2cm, height=2cm]{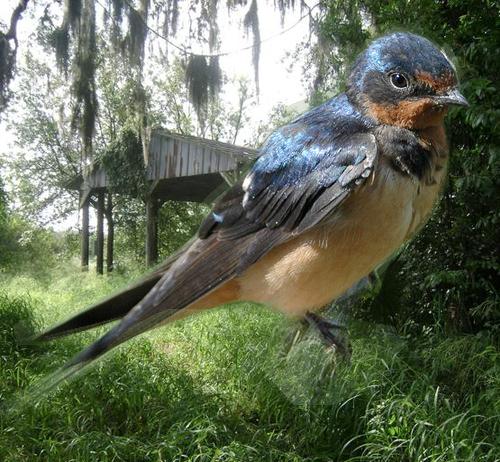} 
& \includegraphics[width=2cm, height=2cm]{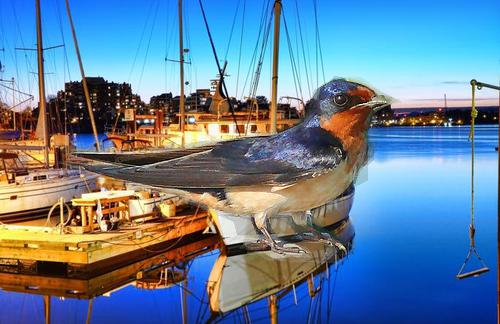} 
& \includegraphics[width=2cm, height=2cm]{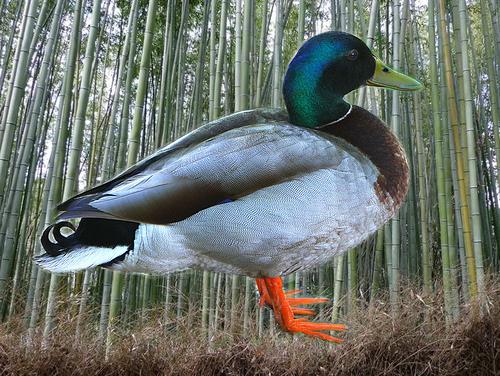} 
& \includegraphics[width=2cm, height=2cm]{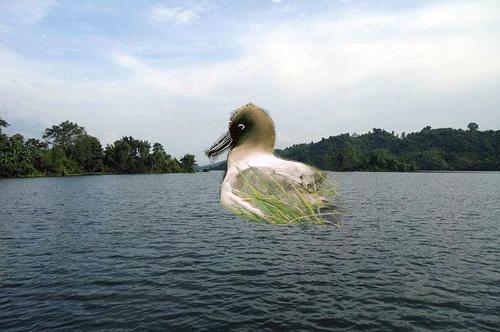} \\
\midrule
ClassLabel      & 0     & 0     & 1     & 1     \\
AttributeLabel  & 0     & 1     & 0     & 1     \\
GroupLabel      & 0     & 1     & 2     & 3     \\
\midrule
\#TrainingData   & 3,498 & 184   & 56    & 1,057 \\
\#ValidationData & 467   & 466   & 133   & 133   \\
\#TestData       & 2,255 & 2,255 & 642   & 642   \\
\bottomrule
\end{tabular}
}
\end{table}
Regulatory uncertainty over baseline behaviour is modelled via a compact credal set consisting of ERM-trained ResNet-50~\citep{he2016deep} model, mixed with uniform noise to form a credal set of distribution not obedient to regulation 
\begin{equation*}
    \mathcal{P}_0=\{P\quad|\quad \epsilon P_{ERM}+(1-\epsilon)P_{uniform}\quad \epsilon\in[0,1]\}
\end{equation*}
Where $P_{uniform}$ is just randomised prediction for $P(Y|X=x)$, i.e. randomly says if a bird is waterbird or land bird. Effectively, $\mathcal{P}_0$ represents the mixture of distributions which rely on the spurious features, background information in the case of ERM and and random noise in the case of $P_{uniform}$.

\subsection{Details of the Fairness Regulation Experiment} 
\label{appendix:experiment-details-fairness}

We now discuss the implementation of the bets for fairness regulation. We consider paired data for both subgroups from the distributions $Y_0\sim \text{Bernoulli}(0.1)$ and $Y_1\sim \text{Bernoulli}(\Gamma+0.1)$ for $\Gamma\in\{0.3,0.6\}$. We now show that our test statistic $\pi_n = \prod_{t=1}^n (1 + \lambda_t (\tau-|Y_0-Y_1| ))$ is a test super-martingale for the implicit credal set defined by our requirements. Then from Proposition~\ref{proposition:matringale-implementable} we can argue that the regulation mechanism $\Pi:=\{\pi[\lambda]\}_{\lambda\in\Lambda}$ will be obedient to regulation. Let us denote the implicit credal set as
\[
\mathcal{P}_0 := \{P \mid \Gamma \geq \tau\}.
\]
We show that the wealth process
\[
\pi_n = \pi_{n-1}\Big(1 + \lambda_n(\tau - |Y_0 - Y_1|)\Big), \qquad \pi_0 = C,
\]
is a non-negative supermartingale under every $P \in \widetilde{\mathcal{P}}_0 := \{P \mid \mathbb{E}_P|Y_0 - Y_1| \geq \tau\}$. By Proposition~\ref{proposition:matringale-implementable}, this implies that the regulation mechanism $\Pi := \{\pi[\lambda]\}_{\lambda \in \Lambda}$ is obedient on $\widetilde{\mathcal{P}}_0$. Let $\mathcal{F}_n = \sigma(\{Y_{0}^i\}_{i=1}^n, \{Y_{1}^i\}_{i=1}^n)$. For any $P \in \widetilde{\mathcal{P}}_0$,
\begin{align*}
    \mathbb{E}_P[\pi_n \mid \mathcal{F}_{n-1}]
    &= \mathbb{E}_P\!\left[\pi_{n-1}\big(1 + \lambda_n(\tau - |Y_{0} - Y_{1}|)\big) \,\Big|\, \mathcal{F}_{n-1}\right] \\
    &= \pi_{n-1}\Big(1 + \lambda_n\big(\tau - \mathbb{E}_P[|Y_{0} - Y_{1}| \mid \mathcal{F}_{n-1}]\big)\Big)
       \tag*{(\mbox{$\pi_{n-1}, \lambda_n$} are \mbox{$\mathcal{F}_{n-1}$}-measurable)} \\
    &= \pi_{n-1}\Big(1 + \lambda_n\big(\tau - \mathbb{E}_P|Y_0 - Y_1|\big)\Big)
       \tag*{(i.i.d.\ sampling)} \\
    &\leq \pi_{n-1},
\end{align*}
where the last step uses $\lambda_n \geq 0$ and $\mathbb{E}_P|Y_0 - Y_1| \geq \tau$ (since $P \in \widetilde{\mathcal{P}}_0$). Hence $\pi_n$ is a non-negative supermartingale under every $P \in \widetilde{\mathcal{P}}_0$, so by Proposition~\ref{proposition:matringale-implementable}, $\Pi$ is obedient on $\widetilde{\mathcal{P}}_0$.

\paragraph{Relationship between $\widetilde{\mathcal{P}}_0$ and $\mathcal{P}_0$.}
By Jensen's inequality, $\mathbb{E}_P|Y_0 - Y_1| \geq |\mathbb{E}_P[Y_0] - \mathbb{E}_P[Y_1]| = \Gamma$. Therefore $\{P \mid \Gamma \geq \tau\} \subseteq \{P \mid \mathbb{E}_P|Y_0 - Y_1| \geq \tau\}$, i.e., $\mathcal{P}_0 \subseteq \widetilde{\mathcal{P}}_0$. The mechanism is thus obedient on the (possibly strict) superset $\widetilde{\mathcal{P}}_0 \supseteq \mathcal{P}_0$, which preserves obedience on $\mathcal{P}_0$ but may sacrifice feasibility for borderline-compliant providers: those whose true fairness gap satisfies $\Gamma < \tau$ but for whom $\mathbb{E}_P|Y_0 - Y_1| \geq \tau$ are treated as non-compliant. This is the source of the self-exclusion observed at $\Gamma = 0.5$.

\section{Challenges in AI Regulations Beyond Statistical Issues}
\label{appendix-non-technical-challenges-for-regulation}  

Statistical or technical challenges set aside, AI regulation has several non technical challenges compared to classic regulations in the past as there are seldom any goods or process that are as general as ``intelligence'' and have such close human interaction. One key issue is that the liability of AI model's risk is fragmented across model designers, data suppliers, integrators, and deployers, complicating enforcement \citep{tabassi2023artificial, bertolini2021expert}.  
Another aspect is of Jurisdictional fragmentation and cross-border deployment, which undermine coherent remedies and legal actions on designers or other stake holders\citep{edwards2021eu,uk_frontier_ai_2023}.  
There are no widely adopted technical standards or certification regimes; proprietary intellectual-property and trade secrets conflict with transparency and auditability \citep{raji2020closing}.   
Supply-chain opacity in data provenance, labelling, and collection prevents reliable forensics also offer some additional challenges\citep{bender2021dangers}.  
Market concentration of some large scale service providers also known as ``big-tech'' in compute and data creates political-economy pressures and regulatory capture \citep{lohn2022much,korinek2025concentrating}. 
Dual-use capabilities, adversarial gaming, and benchmark overfitting lets actors satisfy narrow tests while retaining harmful capacity \citep{blum2015ladder,mazeika2024harmbench,hardt2025emerging}. These threats are further exacerbated by test time adaptation of AI models and lack of strong defences for these cases~\citep{wu2023uncovering,singh2023robust}.
Often evidence standards in courts and agencies are immature for probabilistic, high-dimensional technical proofs \citep{kroll2015accountable}.  
Certification and continuous audit impose high fixed costs that raise market-entry barriers \citep{raji2020closing}.  
Human-in-the-loop requirements are hard to specify and brittle in practice \citep{amodei2016concrete}.  
Finally, cultural and ethical pluralism, privacy trade-offs in monitoring, and systemic risks from correlated deployments mean regulation must reconcile competing values under uncertainty \citep{unesco2022recommendation} and require new paradigms where such pluralism is baked in~\citep{singh2024domain}.  
These legal, economic, organizational, and security frictions interact with the information asymmetry and statistical uncertainty, to make AI regulation both harder to formulate and easier for stakeholders to evade than conventional regulations \citep{brundage2018malicious}.

\section{Broader Impact}
\label{app:broader-impact}
This work develops a theoretical framework for designing AI regulations that account for both information asymmetry between regulators and providers, and the statistical uncertainty inherent in evaluating compliance from finite samples. By characterising when perfect market outcomes are achievable in terms of credal sets, our results provide policymakers with a diagnostic tool to assess whether a proposed regulation can, in principle, be enforced without admitting non-compliant providers and excluding compliant ones. This has several potential positive consequences. First, it offers a principled basis for choosing between candidate regulatory metrics: requirements such as worst-case subgroup accuracy or sub-group fairness, while normatively appealing, induce non-credal sets of non-compliant distributions and are therefore vulnerable to strategic gaming through randomisation. Surfacing this trade-off may help regulators select surrogate metrics that are both enforceable and aligned with their underlying normative goals, or to consciously adopt conservative convex-hull regulations when no credal surrogate is available. Second, our connection to hypothesis testing and testing-by-betting shows that existing statistical machinery can be repurposed to construct implementable mechanisms, lowering the technical barrier for regulatory bodies to operationalise black-box regulations.

We also recognise potential limitations and risks. Our framework assumes rational, expected-utility-maximising providers; in practice, providers may be risk-averse, boundedly rational, or coordinate in ways our single-agent model does not capture, and the incentive guarantees of obedience and feasibility hold only ex ante. Mis-specifying the credal set $\mathcal{P}_0$, whether through poorly chosen surrogates or implicit definitions that fail to track the regulator's actual concern, could produce regulations that are technically PMO-achieving but normatively misaligned, lending false legitimacy to weak compliance standards. There is also a risk that the framework's apparent rigour is invoked to justify regulations whose real bottleneck is political or institutional rather than statistical, including the legal, jurisdictional, and supply-chain frictions discussed in Appendix \ref{appendix-non-technical-challenges-for-regulation}. Finally, by shifting the burden of proof onto providers, mechanisms of the kind we study may raise market-entry costs and reinforce concentration among well-resourced incumbents, an outcome that regulators should weigh against the gains in compliance. We see our contribution as one input into a broader, necessarily interdisciplinary effort, and encourage its use alongside, rather than in place of, qualitative policy analysis and stakeholder engagement.
\end{bibunit}

\end{document}